\documentclass[lettersize,journal]{IEEEtran}
\usepackage{amsmath,amsfonts}
\usepackage{algorithmic}
\usepackage{array}
\usepackage[caption=false,font=normalsize,labelfont=sf,textfont=sf]{subfig}
\usepackage{textcomp}
\usepackage{stfloats}
\usepackage{url}
\usepackage{verbatim}
\usepackage{graphicx,cite}
\usepackage{booktabs,amsthm}
\usepackage{balance}
\newtheorem{theorem}{Theorem}[section]

\def\BibTeX{{\rm B\kern-.05em{\sc i\kern-.025em b}\kern-.08em
    T\kern-.1667em\lower.7ex\hbox{E}\kern-.125emX}}

\begin{document}
\title{Defocus Aberration Theory Confirms Gaussian Model in Most Imaging Devices}
\author{Akbar Saadat \thanks{The author works with the R\&D Department  of Iranian railways (RAI). This work has developed as an update to his academic research on ''Depth Finding by Image analysis'' since 1995.\\E-mail: saadat\_a@rai.ir, Tel: (+98)9123840343}}

\maketitle
\markboth{Akbar Saadat, Under Review Process by  Manuscript ID TPAMI-2025-02-0453. }{Main Manuscript}

\begin{abstract}
Over the past three decades, defocus has consistently provided groundbreaking depth information in scene images. However, accurately estimating depth from 2D images continues to be a persistent and fundamental challenge in the field of 3D recovery.  Heuristic approaches involve with the ill-posed problem for inferring the spatial variant defocusing blur, as the desired blur cannot be distinguished from the inherent blur. Given a prior knowledge of the defocus model, the problem become well-posed with an analytic solution for the relative blur between two images, taken at the same viewpoint with different camera settings for the focus. The Gaussian model stands out as an optimal choice for real-time applications, due to its mathematical simplicity and computational efficiency. And theoretically, it is the only model can be applied at the same time to both the absolute blur caused by depth in a single image and the relative blur resulting from depth differences between two images. This paper introduces the settings, for conventional imaging devices, to ensure that the defocusing operator adheres to the Gaussian model. Defocus analysis begins within the framework of geometric optics and is conducted by defocus aberration theory in diffraction-limited optics to obtain the accuracy of fitting the actual model to its Gaussian approximation. The results for a typical set of focused depths between  $1$ and $100$ meters, with a maximum depth variation of $10\%$ at the focused depth, confirm the Gaussian model's applicability for defocus operators in most imaging devices. The findings demonstrate a maximum Mean Absolute Error $(\!M\!A\!E)$ of less than $1\%$, underscoring the model's accuracy and reliability.
\end{abstract}


\section{Introduction}
\IEEEPARstart{C}{omputer} vision is the only solution for making an active interaction between a machine and its environment to control. It deals with two-dimensional images of a scene as input to extract the third dimension or depth at each image point as output. The earliest solutions to the problem of obtaining depth were geometric-based methods such as stereo vision (\cite{ZLCa2015, URDh1989}) and structure from motion (\cite{SULI1979, DAFo2003}). These methods have been the subject of extensive investigation over the past four decades. The precision of depth estimation is contingent upon the accurate correspondence \cite{DSch2002} between the images, which is fundamental to the geometric algorithms grounded in triangulation. Due to the method's computational intensity and susceptibility to significant errors, the research communities were motivated to investigate alternative solutions to the correspondence problem.

Since the initial introduction of focal gradients by A.P. Pentland \cite{APPe1987} as a novel source of depth information, the accurate estimation of depth from two-dimensional images—without requiring direct correspondence—has persisted as a significant challenge in the domain of three-dimensional recovery. In the context of capturing images with a limited depth of field, the occurrence of defocus blur is inevitable. This is caused by the scene points being out of focus or shifted away from the camera's focal plane within the scene. The amount of shift is directly correlated with the depth of the scene points, according to the geometric optics. Throughout the past three decades, numerous methodologies have been proposed to address the Depth From Defocus (DFD) problem \cite{APPe1987, MSub1988,MSub1990,YXio1993, Jens1993,SKNa1994,MSub1994,ANRa1997,ANRa1998,MWat1998,ANRa1999,DRaj2003,ANRa 2004}. These researchers are in the group of scientists who have made the most significant contributions to the DFD techniques, established foundation for subsequent advancements and introduced several mathematical models and algorithms that have influenced the development of the field. 

The absence of an end proof model for the defocus operator has led DFD methods being associated with the integration of sparse and heuristic features, such as image gradients or moments to estimate depth. The foundation on heuristics and assumptions about depth variation within the scene has constrained the application of DFD techniques in situations where estimating depth information from alternative sources, such as stereo vision, is either challenging or unavailable. While DFD is not inherently less sensitive to error than stereo or motion, it is more robust due to the 2D nature of the aperture. Consequently, it should be preferred over small baseline stereo if the resolution, obtainable with DFD implementations, is sufficient \cite{YYSc1998}. Research collectives explored hardware modifications initiated by \cite{ALev2007} to enhance resolving power with synthetic apertures, replacing traditional image formation kernels with custom-oriented designs. This approach aimed to better characterize image formation and facilitate depth inference from defocus. However, deriving the image formation model required extensive local camera calibration and the use of point sources or predetermined high-frequency patterns at various potential depths \cite{MDel2012}. This posed a significant challenge for advancing the DFD method within research groups focused on three-dimensional recovery, as achieving the necessary calibration was difficult without relying on additional sources or patterns.

Research communities commissioned themselves to touch the human skills on inferring 3D structure or depth from a single image. On the early steps, the public sentiment towards the heuristic features, such as image gradients or moments, reverted to the key point descriptors \cite{MHas2019} such as speeded-up robust features (SURF) \cite{HBay2006}, pyramid histogram of oriented gradient (PHOG) \cite{Abos2007}, scale invariant feature transform (SIFT) (\cite{DGlo1999, DGLo2004}), and probabilistic graphs such as Conditional Random Field (CRF)\cite{JLaf2001} and Markov Random Field (MRF) \cite{GRCr1983}. These features were considered for depth estimation in a single image with parametric \cite{ASax2009}and non-parametric (\cite{CLiu2009,BLiu2010}) machine learning procedures. For inferring depth from a single image a comprehensive database of the world images is required with their 3-D coordinates. This is highlighted in \cite{BCRu2009} for integration all present RGB Depth databases.  Regardless of how far away is, \cite{BLiu2010} has exploited the availability of a pool of images with known depth to formulate depth estimation as an optimization problem. The research area which was tightened by enforcing geometric assumptions to infer the spatial layout of a room in \cite{VHed2010} and \cite{DCLe2010} or outdoor scenes in \cite{AGup2010}, was expanded by handcrafted features in \cite{ASax2009, BLiu2010, KKar2014, LLad2014}  and \cite{MLiu2014} for more general scenes. 

By the emergence of deep learning architectures all responsibilities for feature extraction, feature detection and  mapping features to depth delivered to the multi layers of Convolutional Neural Networks (CNN) \cite{FLiu2016,FLiu2015,DEig2014},  which  infer directly depth map from the image pixel values. In this approach, there is no basic difference between depth estimation and semantic labelling, as jointly performing both can benefit each other \cite{LLad2014}. The potential to create semantic labels, which provide meaningful annotations to different parts of an image, can significantly enhance depth perception. This capability is particularly useful in guiding the depth estimation process. By incorporating semantic information, the system can better understand the context and relationships within the image, leading to more accurate depth predictions. This approach supports the validity and effectiveness of using a multistage inference process in CNNs. The generation of semantic labelling with the objective of guiding depth perception in \cite{BLiu2010} is an effective realization of that support.

CNNs have their own Limitations in 3D recovery of a scene image. With precisely calibrated architectures and hyper parameters, they can learn features from the training set from scratch during the training period. They cannot do anything more than extrapolation for what is beyond this limited space, even if given infinite time to completely learn the training set. Literature has documented several shortcomings of CNNs. For instance, \cite{PWan2016} points out that CNNs often fail to ensure their predictions align with the planar regions depicted in the scene. Additionally, existing CNN architectures (e.g., VGG-16 \cite{KSim2015}) can not predict good surface orientations from depth, and pooling operations and large receptive fields makes current architectures perform poorly near object boundaries \cite{QiXL2020}. And, in order to succeed in challenging image regions, such as areas near depth discontinuities, thin objects and weakly textured zones, it is necessary to learn a broad range of principles and features that limit the possibility of focusing on important details \cite{WIqb2023}. To address the aforementioned limitations, these researchers devised a combination of conventional hand-crafted and deep learning-based methods, collectively referred to as hybrid techniques. In these methods, the predictions of deep networks are refined by the features extracted from the input image. Furthermore, the results of deterministic features, such as edges, in the sparse locations where the features are available, are replaced with the networks' predictions based on an image formation model.

Setting a model for image formation with a proven relationship to depth allows the reconstruction of a dense depth map entirely by a hand-crafted feature with no reliance on deep networks. The research conducted in \cite{ASaa2017} indicated that a hand-crafted feature with the Gaussian model for the defocusing operator could offer superior performance compared to learned features, including those derived from deep learning. The advantages of computational and analytical simplicity in both the frequency and spatial domains for real-time applications make the Gaussian model highly appealing for DFD. Theoretically, it is the unique model that can simultaneously address both the absolute blur caused by depth in a single image and the relative blur resulting from depth differences between two images. Dealing with the aim to contribute the current theory of DFD, this paper focuses on the role of diffraction-limited optics in validating the Gaussian model for the defocus operator in conventional imaging systems. The study reveals how diffraction-limited systems support the Gaussian approximation of the defocus operator. While the paper does not deal with any experiment for the contribution, it highlights the capability of the conventional imaging systems to support the Gaussian model in a wide range of depth finding and under  mild conditions on general settings.

Paper organization for driving the settings for a general imaging device to ensure that the defocusing operator conforms to the Gaussian model is as follows. Section 2 highlights general aspects of the existing theory of DFD in relation to image formation models.  The defocusing Optical Transfer Function (OTF) at a single wavelength is illustrated in Section 3, while its characteristics are discussed in Section 4. The OTF under ambient illumination and its Gaussian approximation are detailed in Section 5. In Section 6, the investigation of general settings of conventional imaging systems is conducted under reasonable thresholds of approximation error, with the introduction of those that confirm the Gaussian Model. In the final section of the paper, the key findings and future goals are outlined.

\section{General Aspects Of DFD Theory}
\noindent DFD theory is derived from the image formation model in geometric optics, which ignores the wave nature of light and treats it as rays. The imaging system in geometric optics is characterized by the parameters: $A$ for lens diameter or aperture, $f$ for focal length and $d_i$ for the distance of the image plane to the lens. In the model, all rays parallel to optical axis converge to the focal point and all rays emerge from a single scene point on the focused plane illuminate the image point on the image plane, as shown in Fig.\ref{Fig1} by green and blue colors. The fundamental equation of thin lenses describes the focusing distance $d_f$ in the scene by (\ref{LensLaw}).
\begin{equation}
\label{LensLaw}
d_f = \frac{f d_i}{d_i-f}
\end{equation}

\begin{figure}[!t]
\centering
\includegraphics[width=2.5in]{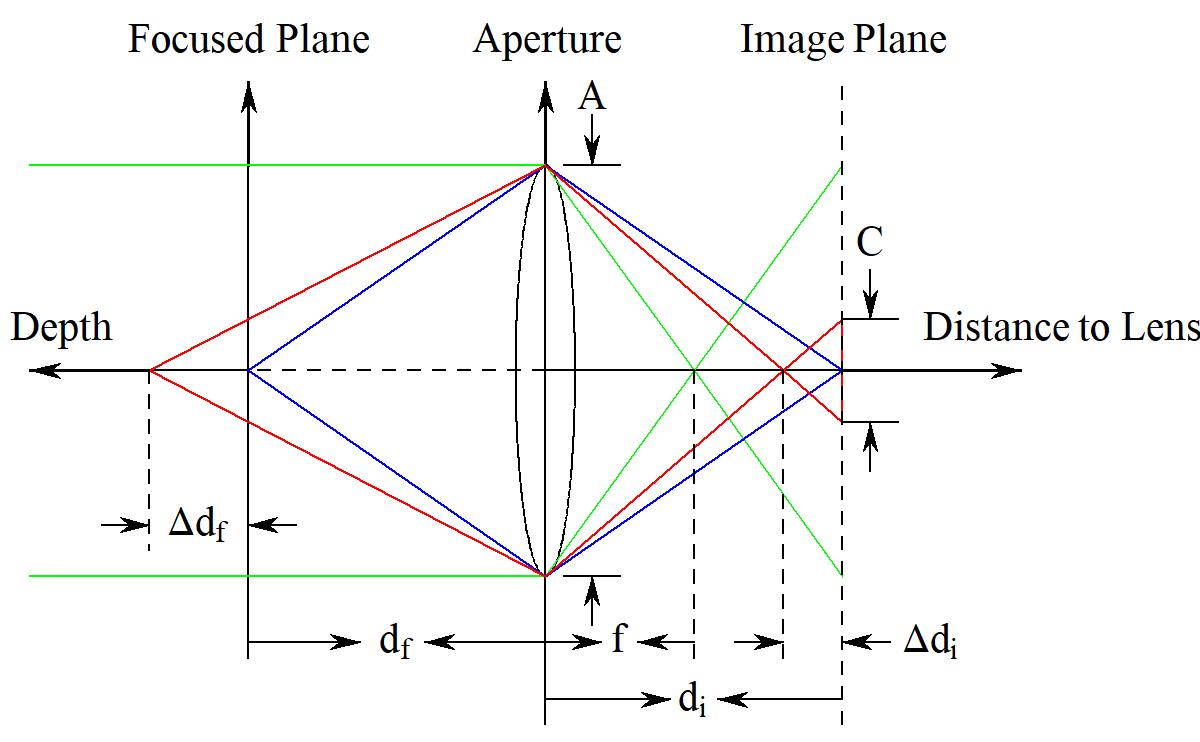}
\caption{Image Formation in geometric optics. Both triples $(f,d_i,d_f)$ and $(f,d_i+\Delta d_i,d_f+\Delta d_f)$ are described by the lens law.}
\label{Fig1}
\end{figure}

Any deviation of the scene point from the focal plane with the amount of $\Delta d_f$ results in a shift of the focused image  from the image plane by the amount of $\Delta d_i$. Again, the new position pairs $d_f+\Delta d_f$  for the scene point and  $d_i+\Delta d_i$ for the focused image are related by the thin lens law as  (\ref{Eq2}).
\begin{equation}
\label{Eq2}
d_f +\Delta d_f = \frac{f (d_i+\Delta d_i)}{d_i+\Delta d_i-f}
\end{equation}
This shift spreads the image point to the blur circle, designated as the Circle of Confusion (CoC) with a diameter C determined by similar triangles as (\ref{Eq3}).
\begin{equation}
\label{Eq3}
 \frac{C}{A} = \frac{\Delta d_i}{d_i+\Delta d_i}
\end{equation}
The signs of $\Delta d_i$ and $\Delta d_f$ are derived from the directions of the horizontal axes, designated as "Depth" and "Distance to Camera Lens," within rectangular coordinate systems with vertical axes positioned at the focal and aperture planes, as illustrated in Fig.1.  Eliminating $\Delta d_i$    between  (\ref{Eq2}) and  (\ref{Eq3}) relates depth of the scene point $\Delta d_f$ to the blur circle $C$ through the settings parameters by (\ref{Eq4}).
\begin{equation}
\label{Eq4}
\Delta d_f  = \frac{C d_f}{C_o-C}  \hspace{1cm}  C_o=\frac{A f}{d_f-f} 
\end{equation}

The theory of DFD is completed by introducing four equations between four unknown variables on two images of a point in the scene to obtain its depth. These variables are two pairs of the distance of the image point to the focused plane and the diameter of the blur circle, in two images. The first equation is the depth map as  (\ref{Eq4}) that describes the relation between the distance of the image point to the focused plane and the diameter of the blur circle in one image. The second equation is the same as the first, but for the other image. The third equation expresses the fact that absolute difference of the depth values is equal to the known distance between the focused planes.  The last equation figures out the relation between the diameter of the blur circles in two images.

When the defocus operator aligns with the Gaussian model, the size of each blur circle will be proportional to the model’s standard deviation. (This proportionality will be confirmed through the subsequent sections.) Moreover, the defocus operator that transforms the sharper image into the second one follows a Gaussian model. The fourth equation sets the standard deviation of this model equal to the square root of the difference between the squares of the standard deviations of the defocusing operators.Therefore, the image computation in DFD involves extracting the standard deviation of the defocusing model at a specific image point from a sharper image.This can be achieved through arbitrary sequence of convolutions to revisit the image point without requiring windowing precautions. The prerequisite for this process is that the defocus operator conforms to the Gaussian model. Consequently, the primary focus of this paper is to evaluate the characteristics of the analytic solution of the defocusing operator in diffraction-limited optics to verify the Gaussian model.

The solution is based on the Intensity Image Response (IIR) of the imaging system shown in Fig.\ref{Fig1} when exposed to a point source. In the spatial domain, the IIR is referred to as the Point Spread Function (PSF), and in the frequency domain, it is known as the Optical Transfer Function (OTF). For diffraction-limited conventional imaging systems, the PSF is typically obtained by modeling the real optical system of several lens layers to determine the optical path difference using ray tracing software simulations. According to \cite{FSon2024}, the root mean square errors of both the analytical and simulated methods were found to be less than 3 percent in the weak and medium defocus range for three examples of monochrome conventional imaging systems. This suggests that the analytic solution for the system in Fig.\ref{Fig1} is a reasonable approximation for conventional imaging systems.

\section{Monochrome OTF for\\ Defocused Imaging System}
\noindent In contrast to geometric optics, the image of a focused scene point is not a single point, and the CoC is not a uniformly bright disk in the image plane. Looking in the diffraction limited optics, the solution in the spatial frequency domain appears in the case of monochromatic light illuminating at the wave length $\lambda$. In this case, the IIR of a non-coherent imaging system to a point source of an object at the focused distance is given \cite{JWGo2005} by the OTF $H(f_x,f_y )$ as (\ref{Eq5}).
\begin{align}
&H(f_x,f_y;P)  =  \label{Eq5}\\
&\frac{\iint\limits_{-\infty}^{+\infty}P(x+\frac{\lambda d_if_x}{2},y+\frac{\lambda d_if_y}{2})P^*(x-\frac{\lambda d_if_x}{2},y-\frac{\lambda d_if_y}{2})dxdy}{\iint\limits_{-\infty}^{+\infty}|P(x,y|^2dxdy} \nonumber
\end{align}
$(f_x,f_y )$ is the spatial frequency pairs related to the $(x,y)$ spatial domain. $P(x,y)$ is the pupil function which is unity for an aberration free system within the aperture and zero otherwise.  $P^*$ is generally the complex conjugate of $P$.

The effects of aberrations caused by defocus is generating a phase shift for the wave front that leaves the pupil.   If the phase shift at the point $(x, y)$ is expressed by $\kappa W(x,y)$ for the phase number $\kappa=2\pi/\lambda$, then with the effective path-length error $W(x,y)$ and $j^2=-1$, the complex aperture would be $P_{def} (x,y)=P(x,y)exp(j\kappa W(x,y))$. The path-length shift $W(x,y)$ is related  \cite{JWGo2005} to the given parameters in Fig.\ref{Fig1} by  (\ref{Eq6}).
\begin{equation}
\label{Eq6}
W(x,y)=\frac{-1}{2}(\frac{1}{d_i+\Delta d_i}-\frac{1}{d_i})(x^2+y^2)\stackrel{\triangle}{=}\frac{A_R}{R^2}(x^2+y^2)
\end{equation}
The number $A_R $ is the maximum shift at the boarder of the aperture where $\sqrt{x^2+y^2}=R=A/2$.  This number indicates the degree of defocusing effects in diffraction limited optics, and sounds the size of the CoC, the same indicator in geometric optics. There is a linear relationship between $A_R$ in (\ref{Eq6}) and $C$ in (\ref{Eq3}) as outlined in (\ref{Eq7}).
\begin{equation}
\label{Eq7}
\frac{A_R}{C}=\frac{-R^2}{2}(\frac{1}{d_i+\Delta d_i}-\frac{1}{d_i})\frac{d_i+\Delta d_i}{A\Delta d_i}=\frac{R^2}{2Ad_i}=\frac{A}{8d_i}
\end{equation}
This completes the definition of the elements in the expression of the OTF for the defocused imaging system given by  (\ref{Eq8}).
\begin{equation}
\label{Eq8}
H_{def}(f_x,f_y)  = H(f_x,f_y;P_{def})
\end{equation}
 
For the circular aperture with the radius $R$, described by the unity disk function $circ()$ as $P(x,y)=circ(\sqrt{x^2+y^2}/R)$, both 2-D functions $H(f_x,f_y )$  and $H_{def} (f_x,f_y )$ in the rectangular coordinates are circularly symmetric with 1-D functional forms $H^o (\rho)$ and $H_{def}^o (\rho)$ in the polar coordinates with $\rho^2=f_x^2+f_y^2$. Considering  the coherent cut off frequency \cite{JWGo2005}  $\rho_o=R/(\lambda d_i )$  and the definition $\theta \stackrel{\triangle}{=}\arcsin(\rho/2\rho_o )$ for $\rho\leq2\rho_o$, the expression derived in literature from $H(f_x,f_y ;P)$ for $H^o (\rho)$ can be simplified to  (\ref{Eq9}).
\begin{equation}
\label{Eq9}
H^{o}(\rho)=\frac{2}{\pi}
\begin{cases}
\theta-\frac{1}{2}\sin(2\theta) & \rho\leq2\rho_o \\
0&  \rho>2\rho_o
\end{cases}
\end{equation}
In the case of square pupil, the function $H_{def} (f_x,f_y )$ is expressed as a product of two identical one-dimensional functions, each of which is a function of $f_x$ and $f_y$, respectively. But, there is no given expression in literature for $H_{def}^{o} (\rho)$ in case of circular pupil. The expression for $H_{def}^{o} (\rho)$ is derived from $H_{def} (f_x,f_y )$ in the Appendix \ref{AppendixEq10} as  (\ref{Eq10}).
\begin{align}
&H_{def}^{o}(\rho)  =\frac{4}{\pi} \label{Eq10} \\
&
\begin{cases}
\int\limits_{0}^{1-\cos(\theta)}\sqrt{1-(x+\cos(\theta))^2}\cos(8\pi \frac{A_R}{\lambda}x\cos(\theta))dx & \rho\leq2\rho_o \\
0&  \rho>2\rho_o
\end{cases}  \nonumber
\end{align} 
The parameter $A_{R}/\lambda$  in the given expression for $H_{def}^{o}(\rho)$ represents directly the amounts of defocusing aberrations. It could be verified that $H_{def}^{o}(\rho)$ is  equal to $H^{o}(\rho)$for the aberration free system in which $A_R/\lambda$ is identically zero. For various values of $A_{R}/\lambda$ plots of $H_{def}^{o}(\rho)$ are shown in  Fig.\ref{Fig2}. It is interesting to compare it with the cross section of OTF for square pupil  \cite{JWGo2005} nominated by $H_{def}^{G}(\rho)$  in Fig.\ref{Fig3}. While in both plots spatial high frequencies are attenuated more naturally at higher values of $A_{R}/\lambda$, the circular pupil has a band pass frequency wider than the square pupil for the same non zero values of $A_{R}/\lambda$. By increasing $A_R/\lambda$ the location of the first zero appears at the same normalized frequency $0.5$ on the horizontal axis. This occurs at slightly higher degree of defocus ($A_{R}/\lambda\approx 0.64$) for the circular pupil than ($A_{R}/\lambda=0.5$)   for the square pupil.
\begin{figure}[!t]
\centering
\includegraphics[width=2.5in]{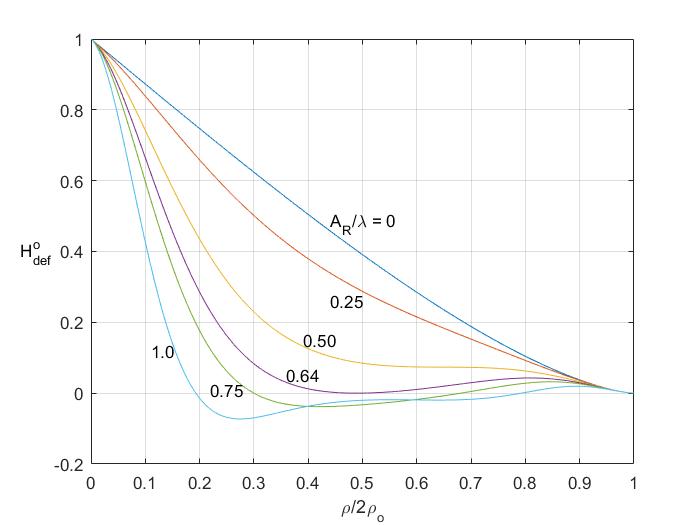}
\caption{OTF for the defocused imaging system with $A_{R}/\lambda$ as a parameter.  Circular pupil with the diameter $A=2R$.}
\label{Fig2}
\end{figure}
\begin{figure}[!t]
\centering
\includegraphics[width=2.5in]{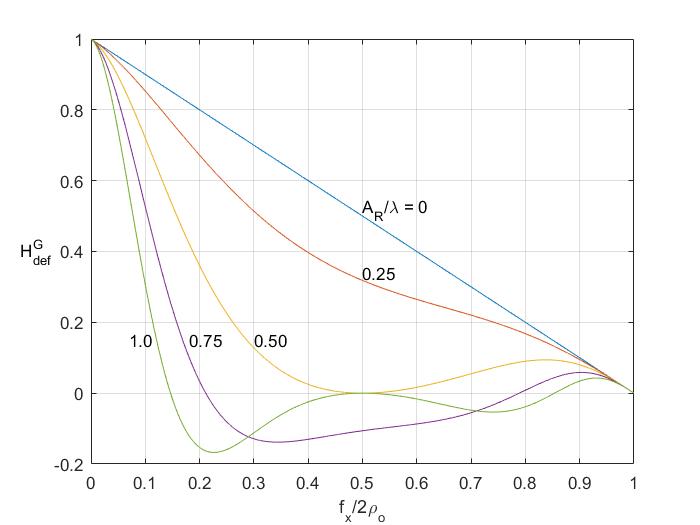}
\caption{OTF for the defocused imaging system with$A_{R}/\lambda$ as a parameter Cross section of OTF for square pupil of width $2R$ along the $f_x$ axis \cite{JWGo2005}.}
\label{Fig3}
\end{figure}
 
Concentrating on the circular pupil for the rest of this paper, the defocusing OTF at any wave length $\lambda$ is obtained by dividing the transfer function of  $H_{def}^{o}(\rho)$ in (\ref{Eq10}) to $H^{o}(\rho)$ in (\ref{Eq9}). This can be evaluated in Fig.\ref{Fig2} by normalizing each  plot’s values to the corresponding values for $A_R/\lambda=0$ for the selected values of $A_{R}/\lambda$.  To approach an analytic solution to the integral in equation (\ref{Eq10}),  the multiplicand squared function of the integrand is replaced with one of the following functions in (\ref{Eq11}) for $k=1,2 \text{ or } 3$, since its approximation generates analytic solution for the integral.
\begin{equation}
 \label{Eq11} 
\sqrt{1-(x+\cos(\theta))^2}  \approx\sin(\theta)(1-(\frac{x}{1-\cos(\theta)})^{k})
\end{equation}
All approximations in  (\ref{Eq11})  are convex functions in the integration interval and have the exact values of the original function at both limits $x=0$ and $x=1-\cos(\theta)$. It is shown in the  Appendix \ref{AppendixEq11} that there is a unique $k(\theta)$ that satisfies this property for the mid point $x=\cos(\theta/2)-\cos(\theta)$ also, and  the mean value of $k(\theta)$ over the range $0<\theta<\pi/2$ is $2.70428$.  The result obtained by applying the simplest approximation for the case $k=1$ in (\ref{Eq10}) is presented in (\ref{Eq12}).
\begin{align}
&H_{def}^{o}(\rho)   \approx \frac{4}{\pi} \label{Eq12} \\
&
\begin{cases}
\sin^{2}(\frac{\theta}{2})\sin(\theta)sinc^{2}(8\frac{A_{R}}{\lambda}\cos(\theta)\sin^{2}(\frac{\theta}{2})) & \rho\leq2\rho_o \\
0&  \rho>2\rho_o
\end{cases}\nonumber
\end{align} 
Now, the defocusing OTF at the wave length $\lambda$ is approximated by  (\ref{Eq13}).
\begin{align}
&H_{def}^{\lambda}(\rho) =\frac{H_{def}^{o}(\rho)}{H^{o}(\rho)}  \approx  \label{Eq13}\\
&
\begin{cases}
\frac{2\sin^{2}(\frac{\theta}{2})\sin(\theta)sinc^{2}(8\frac{A_{R}}{\lambda}\cos(\theta)\sin^{2}(\frac{\theta}{2})) }{\theta-\frac{1}{2}\sin(2\theta) }  & \rho\leq2\rho_o\\
0&  \rho>2\rho_o
\end{cases} \nonumber
\end{align}
This approximation provides clearer and more transparent features than the non-closed integral form of the exact function, making it more effective for characterizing  $H_{def}^{\lambda}(\rho)$ in the following section.

\section{Characteristics of Monochrome OTF}
\noindent The exact value of $H_{def}^{\lambda}(\rho)$  at $A_{R}/\lambda=0$ is unity across the range $\rho\leq2\rho_o$, functioning as an all-pass filter that does not alter the amplitude of any frequency components. The  approximation form of $H_{def}^{\lambda}(\rho)$  is primarily influenced  by the $sinc^{2}()$ function. Other factors serve merely as a monotone decreasing multiplicative term, with a maximum attenuation $15\%$ at the highest spatial frequency. Consequently, the $sinc^{2}()$ function plays a crucial role in shaping the behavior of $H_{def}^{\lambda}(\rho)$. The zero crossings of the function are estimated by the frequencies that set the argument of the $sinc^{2} (l)$ function to  integer values, specifically  $l=1,2,3,\dots$.  Similarly, the locations of the extremum values,  where the function touches its maximum or minimum values, occur at half-integer values, specifically $l=1.5,2.5,3.5,\dots$.  Setting the argument  $ 8\frac{A_{R}}{\lambda}\cos(\theta)\sin^{2}(\frac{\theta}{2})=l$ yields  (\ref{Eq14}) for identifying the locations of  zero crossings and extremum frequencies for the approximation of $H_{def}^{\lambda}(\rho)$.
\begin{align}
&\cos(\theta) =\frac{\rho}{2\rho_{o}} = \frac{1}{2} \pm \sqrt{\frac{1}{4}-\frac{l}{4A_{R}/\lambda}}\label{Eq14} \\
&\text{for}\quad \frac{A_R}{\lambda} > l  \quad \text{and}\quad  l=1,1.5,2,2.5,3,\ldots \nonumber
\end{align}
This equation indicates that the frequencies for both the zero crossings and extremums exist if the defocusing parameter $ \frac{A_R}{\lambda}$ is larger than a certain threshold value. These frequencies appear in pairs, symmetrically positioned around the center frequency within the range $\rho\leq2\rho_o$. If it exists, the fringe period can be approximated by the absolute difference between the first two zero crossings before the mid frequency, which occur at $l=2$ and $l=1$ in  (\ref{Eq14}). This approximation, provided by (\ref{Eq15}), quantifies the spacing between the fringes and helps to understand the periodicity of the pattern.
\begin{align}
\Delta\rho &= 2\rho_{o}\left([ \frac{1}{2} - \sqrt{\frac{1}{4}-\frac{2}{4A_{R}/\lambda}}] - [\frac{1}{2} - \sqrt{\frac{1}{4}-\frac{1}{4A_{R}/\lambda}}]\right)\label{Eq15}\nonumber\\
&\approx \frac{\rho_{o}}{2.38\frac{A_R}{\lambda}-2.88} \quad \text{for }\frac{A_R}{\lambda}>2
\end{align}
As the defocusing measure $A_{R}/\lambda$ increases, the first zero crossing emerges, followed by the appearance of the first fringe pattern before the second zero crossing. The threshold values for the first and second zero crossings are $1$ and $2$. Beyond this threshold, the frequency of the fringe pattern varies monotonically with $A_{R}/\lambda$. Plots of the exact value of $H_{def}^{\lambda}(\rho)$ are shown in Fig.\ref{Fig4}. This figure generally confirms the characterization derived from (\ref{Eq14}) and (\ref{Eq15}) regarding the emergence of zero crossings and fringe patterns. For example, the actual threshold values of $A_{R}/\lambda$ for the appearance of the first zero crossing ($0.64$ in Fig.\ref{Fig4} relates to the value $1.0$ to validate  (\ref{Eq14})) and the first fringe pattern ($1.10$ relates to the value $1.5$ in (\ref{Eq14}))  are sufficiently close to the estimations by mid point $(0+1)/2$ and $(1+1.5)/2$, respectively. Applying larger values for $k$ in (\ref{Eq11}) improves the approximations for $H_{def}^{\lambda}(\rho)$ and yields more precise results. However, this would not effectively facilitate the characterization of $H_{def}^{\lambda}(\rho)$. The current approximation is reasonable for extracting the general features of $H_{def}^{\lambda}(\rho)$.
\begin{figure}[!t]
\centering
\includegraphics[width=2.5in]{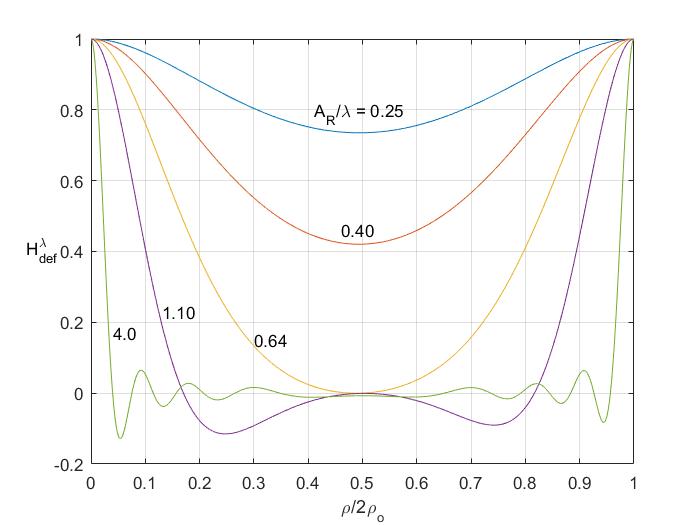}
\caption{Defocusing OTF for defocused imaging system at fixed wave length $\lambda$  and with$A_{R}/\lambda$ as  parameter.}
\label{Fig4}
\end{figure}
\begin{figure}[!t]
\centering
\includegraphics[width=2.5in]{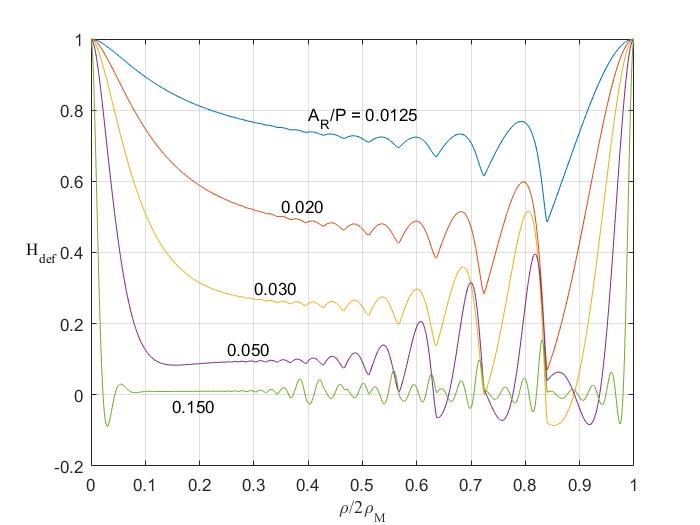}
\caption{Defocusing OTF for defocused imaging system with the pixel width $P=5.6\mu m$ at ambient illumination and maximum coherent cut off frequency $\rho_{M}=\frac{R}{\lambda_{min} d_i } \text{ for } f=15mm, d_f=1m \text{ and } f_n=\frac{f}{2R}=1.4$.}
\label{Fig5}
\end{figure}

\section{OTF Under Ambient Illumination\\ and It's Gaussian Approximation}
\noindent The impact of natural lighting on the imaging device can be characterized by the spectral energy distribution of light across its range of constituent wavelengths. The defocusing OTF under ambient illumination, with the spectral energy distribution $ \phi(\lambda)$ in the range $(\lambda_{min},\lambda_{\!m\!a\!x})$, is described by (\ref{Eq16}).
\begin{equation}
\label{Eq16}
H_{def}(\rho) = \frac{\int_{\lambda_{\!m\!i\!n}}^{\lambda_{\!m\!a\!x}}\phi(\lambda)H_{def}^{\lambda}(\rho)d\lambda}{\int_{\lambda_{\!m\!i\!n}}^{\lambda_{\!m\!a\!x}}\phi(\lambda)d\lambda}
\end{equation}
The ambient illumination is modeled as black body radiation at temperature $T$, that is described by Planck's Law in the MKS system as  (\ref{Eq17}). 
\begin{equation}
\label{Eq17}
\phi(\lambda) = \frac{8\pi h c}{\lambda^5}\frac{1}{e^{h c/(\lambda k_B T)}-1}
\end{equation}
The constants in (\ref{Eq17}) are: Plank constant $h=6.63\times10^{-34}  Js$, Light speed $c=3\times10^8  m/s$ and Boltzmann constant $k_B=1.38\times10^{-23}  J/K$.  To approach the solar spectrum, other parameters in (\ref{Eq17}) are set as  $\lambda_{\!m\!i\!n}=200 nm, \lambda_{\!m\!a\!x}=2 \mu m \text{ and } T=6000^o K$. This temperature is approximately equivalent to the surface temperature of the Sun; therefore, $\phi(\lambda)$ must be scaled for Earth's surface. Nonetheless, $H_{def}(\rho)$ does not require this scaling, as any scaling factor is nullified by (\ref{Eq16}). The elimination of $\lambda$ in $H_{def}(\rho)$  results in the emergence of $ A_R$ as the novel independent parameter. This parameter is expressed in terms of pixel width in the image plane to effectively convey a sense of defocus.

\begin{figure}[!t]
\centering
\includegraphics[width=2.5in]{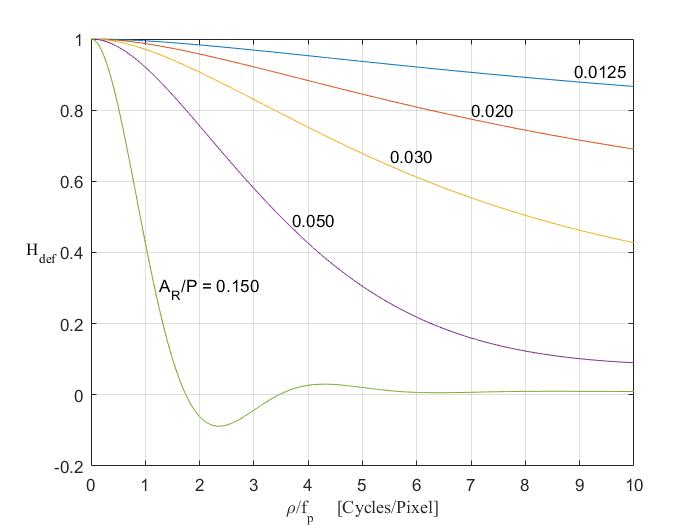}
\caption{Defocusing OTF for starting low frequency part in Fig.\ref{Fig5} versus spatial frequency in terms of cycles per pixel width. There is a monotone bell shape for all plots over the frequency range $\rho<f_P$ that supports Gaussian model.}
\label{Fig6}
\end{figure}     
\begin{figure}[!t]
\centering
\includegraphics[width=2.5in]{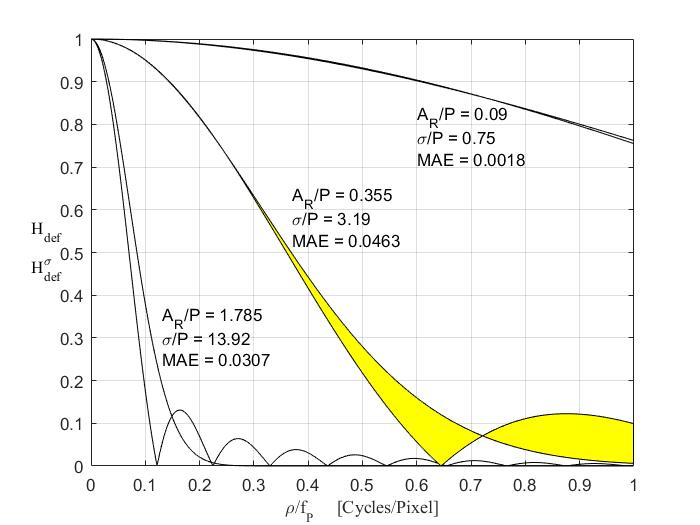}
\caption{Defocusing OTF with $A_R/P$ as  parameter for the practical frequencies limited to one cycle per pixel. The plots are fitted to their approximation Gaussian filters $H_{def}^{\sigma}(\rho)$ with same area under the curve. The Mean Absolute Error $(\!M\!A\!E)$ of fitting is equal to the fraction of the highlighted area between the curves, for the case $A_R/P=0.355$, in the total unit plot area. }
\label{Fig7}
\end{figure}
Considering $P=5.6 \mu m$ as the pixel width and $\rho_M=R/(\lambda_{min} d_i) $ as the maximum value of the coherent cut off frequency, plots of $H_{def}(\rho)$ for various values of $A_R/P$ are shown in Fig.\ref{Fig5} for  the imaging device with $f=15mm, d_f=1m, f_n=f/2R=1.4$.  These plots can be verified by those in Fig.\ref{Fig4} as the later can be considered all with a fixed low value of  $A_{R}/P$ and increasing $\lambda$ from the top to the bottom. The resultant plot can be viewed as  the sum of all plots weighted by $\phi(\lambda)$. For low frequencies, all components contribute to the summation, but $\phi(\lambda)$ emphasizes those with lower $\lambda$ values that have less oscillations. For high frequencies, components beyond the cut off frequencies $2R/(\lambda d_i)$  at low $\lambda$ will be excluded from the summation, leaving only the components with lower coherent cut off frequency at high  $\lambda$ values, which are more oscillatory, to constitute the resultant.

In practical terms, the sensors of the imaging system capture a continuous image of a scene at the pixel cut-off sampling frequency $f_P=1/P$. As a result, the captured image cannot contain frequencies higher than $f_P/2$. This inherent sampling characteristic of the sensors filters out all components with frequencies exceeding half a cycle per pixel in Fig.\ref{Fig5}. Zooming in the low frequency part of the graphs, Fig.\ref{Fig6} illuminates  the plots of defocusing OTF $H_{def}(\rho)$ with respect to the normalised frequency $\rho/f_P$  up to 10 cycles per pixel. Observing a monotone bell shape across all plots over the range of $\rho<f_P$ suggests using a Gaussian model to represent the defocusing OTF in practical applications, especially within the range of  $0<A_R<0.15P$. This range encompasses the near-future conventional cameras that will support super-resolution hardware, effectively doubling the original resolution. To clarify estimation techniques over local images, the range should be mapped to the corresponding range of the variance from the Gaussian model. Defocusing filter in geometric optics is defined as the absolute value of the defocusing OTF, which is already known as Modulation Transfer Function (MTF) in diffraction limited optics.  Fig.\ref{Fig7} illustrates the defocusing filter plots for higher values of the parameter $A_R/P$ compared to those shown in Fig.\ref{Fig6} for $\rho<f_P$. Each plot curve $H_{def}(\rho)$ is fitted to its approximation Gaussian filter $H_{def}^{\sigma}(\rho)=\exp(-\sigma^{2} \rho^{2}/2)$ with the same area under the curve in order to obtain the corresponding $\sigma$. The discrepancy between each curve value and its fitted Gaussian model is quantified by the Mean Absolute Error $(\!M\!A\!E)$ over the plot range. In the case of $A_R/P=0.355$, the measure represents the fraction of the highlighted area in the total  unity plot area. 

The difference between any plot and its Gaussian model is a zero-mean variable over the frequency range $0\leq\rho\leq1$ cycle/pixel. The mean absolute error $(\!M\!A\!E)$ and the analogous well-known measure, the root mean square error ($\!R\!M\!S\!E$), are estimated from the finite samples of this variable. Given that the absolute value of any variable is equal to the square root of its squared, the sole distinction between $\!M\!A\!E$ and $\!R\!M\!S\!E$ is that $\!M\!A\!E$ reverses the order of the mean and square operators in comparison to $\!R\!M\!S\!E$. The mean absolute error ($\!M\!A\!E$) is bounded by the root mean square error ($\!R\!M\!S\!E$) for any finite sequence of numbers, as shown in the Appendix \ref{AppendixDEF}.  The fraction of $\!M\!A\!E$ in $\!R\!M\!S\!E$ is not significantly far less than unity. For example, the values for the zero mean normal and uniform distributions are $\sqrt{2/\pi}$ and $\sqrt{3/4}$, respectively.

\begin{figure}[!t]
\centering
\includegraphics[width=1.7in]{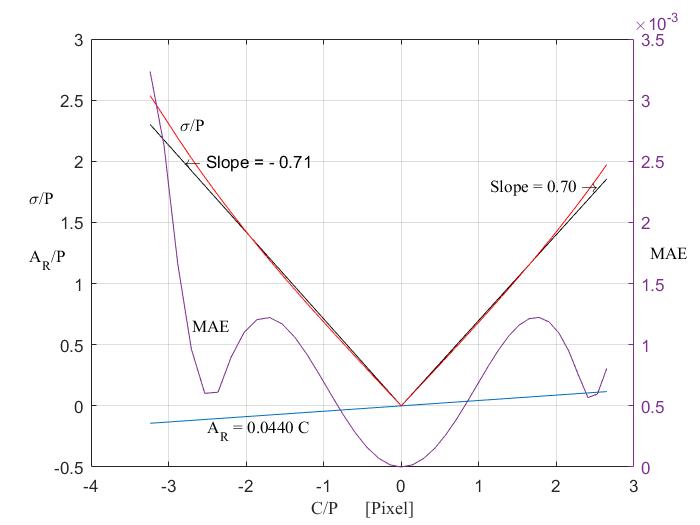}
\includegraphics[width=1.7in]{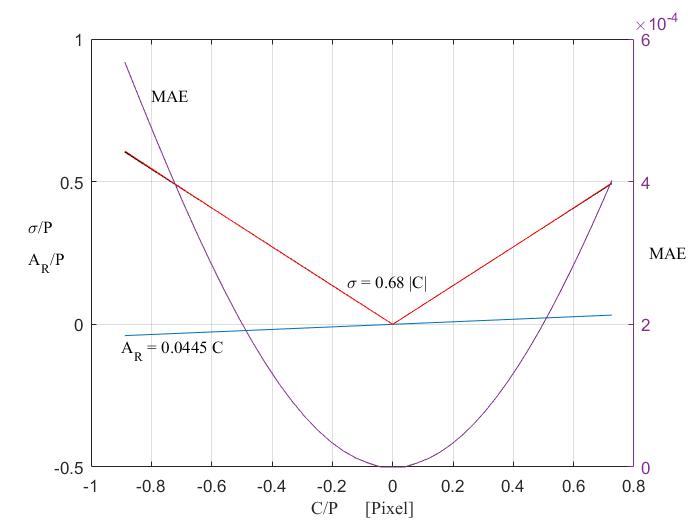}
\caption{Plots of $\sigma, A_R$ and $\!M\!A\!E$ For depth range $-10\%$ to $10\% $ of focal distance by an imaging system with the pixel size $P=5.6 \mu m$ and f-number $f_n=1.4$. Left plot for focal length $f=15 mm$ and focal distance $d_f=1 m$. Right plot for focal length $f=25 mm$ and focal distance $d_f=10 m$.}
\label{Fig8}
\end{figure}

The parameter values in the plots and their models indicate a proportional relationship between $A_R$ and $\sigma$. This relationship, and its connection to the CoC size, is evaluated using an imaging device with typical characteristics: a pixel size of $P=5.6 \mu m$ and an f-number of $f_{n}=f/2R=1.4$, in two scenarios. In one instance, the system with a focal length of $f=15 mm$ is set to the focusing distance of $d_f=1m$. In another instance, the system with a focal length of $f=25mm$ is set to the focusing distance of $d_f=10m$. Based on the known values of  $d_i$ in (\ref{LensLaw}) and $C_o$ in (\ref{Eq4}), the values of $C$ in (\ref{Eq4}) and $A_R$ in (\ref{Eq6}), along with the function $H_{def}(\rho)$ in (\ref{Eq14}), can be determined for a specified depth value $\Delta d_f$  within a certain depth range. For the known $H_{def}(\rho)$, the corresponding values of $\sigma$ and the fitting error $\!M\!A\!E$ for its Gaussian model are available. For the depth range $–d_f/10<\Delta d_f<d_f/10$, Fig.\ref{Fig8} shows the plots of $\sigma, A_R$ and $\!M\!A\!E$ versus the diameter of the blur circle $C$ for both cases. The plots illustrate a linear relationship between $A_R$ and $C$ as described in  (\ref{Eq7}). The graphs of $\sigma$ also demonstrate a satisfactory degree of linearity with $C$, as indicated by the fitted lines. The low $\!M\!A\!E$ values on the right axis of the graphs support the validity of the Gaussian model for defocusing filters. Additionally, the model exhibits remarkable consistency as the focusing distance increases.

\section{General Settings that Confirm Gaussian Model}
\noindent As shown in Fig.\ref{Fig8}, the maximum value of C at the two specified camera settings is a crucial for determining the range of settings that validate the Gaussian model for depth estimation, in terms of mean absolute error $(\!M\!A\!E)$ as the metric. This concept is explored by examining the relationship between the extreme value of $C$ and the settings used. For the relative depth to the focal distance within the range $–\eta<\Delta d_f<\eta$, a straightforward manipulation of (\ref{Eq4}) leads to the relationship expressed in (\ref{Eq18}).
\begin{align}
&C_{\!m\!a\!x} \stackrel{\triangle}{=}  \!m\!a\!x|C| = \label{Eq18} \\
&max \left|\frac{\Delta d_f}{d_f+\Delta d_f}\frac{Af}{d_f - f}\right|=\frac{\eta}{1-\eta} \frac{Af}{d_f - f}\nonumber
 \end{align}
By substituting $f_n=f /A$ into (\ref{Eq18}), the original equation is reformulated into a quadratic function of $f$, as shown in (\ref{Eq19}).
\begin{equation}
\label{Eq19}
f^2+C_m f_n f -C_m f_n d_f = 0, \quad  C_m = \frac{1-\eta}{\eta}C_{\!m\!a\!x}
\end{equation} 
The equation has two distinct real roots with opposite signs. The positive root is given by the expression (\ref{Eq20}). 
\begin{equation}
\label{Eq20}
f = \frac{C_m f_n}{2}\left(\sqrt{1+\frac{4d_f}{C_m f_n}}-1\right)
\end{equation} 

The investigation into the camera settings is confined to the practical discrete ranges of  $f_n\in\{1,1.4,2,2.8,4\}, d_f\in\{1,5,10,20,40.70,100\}  \text{ meters }, C_{\!m\!a\!x}\in\{1,2,3,4,5,6,7\}$ in terms of Pixel width $=P\in\{1,2,4,5.6,8\}  \mu m$ with $\eta=0.1$. In accordance with the triple $(d_f,f_n,C_{\!m\!a\!x})$,  the focal length of the imaging device is determined by $C_m$ in (\ref{Eq19}), and (\ref{Eq20}). The generation of a new triple $(d_f,f_n,f)$  results in the production of three plots for $\sigma$,$A_R$  and $\!M\!A\!E$, over range of variations of  $C$, with the minimum value being $–C_{\!m\!a\!x}$. Two samples from the set of three plots are presented  in Fig.\ref{Fig8} for the following parameter values: $(d_f,f_n,f)=(1m,1.4,15mm)$  and  $(d_f,f_n,f)=(10m,1.4,25mm)$. Each plot reveals the maximum value for $\sigma$ and $\!M\!A\!E$,  designated as $\sigma_{\!m\!a\!x}$ and $\!M\!A\!E_{\!m\!a\!x}$, respectively. The result of total independent quartets $(d_f,f_n,C,P)$ contains $7\times5\times7\times5=1225$ records for the six-element array $(d_f,f_n,C_{\!m\!a\!x},P, \sigma_{\!m\!a\!x},\!M\!A\!E_{\!m\!a\!x})$.

By appropriately discretizing the defocusing filter $H_{def}^{\lambda}(\rho)$, the structure required to store the results can be optimized. $H_{def}^{\lambda}(\rho)$ is a function of two variables, as defined in (\ref{Eq21}).
\begin{equation}
\label{Eq21}
H_{def}^{\lambda}(\rho) = \psi\left(\frac{\rho}{2\rho_o},\frac{A_R}{\lambda}\right) = \psi\left(\frac{\lambda d_i\rho}{A},\frac{AC}{8\lambda d_i}\right)
\end{equation}
The function $\psi$ takes on non-zero values when its first argument is less than unity. The defocusing filter is influenced by the second argument. By setting a uniform discretization with $N$ points, the nonzero values are located at $\rho_k=Ak/(N\lambda d_i )$ for $k=0,1,…,N-1$. Among these, the first $M$ points in the range $\rho<1/P$ are applicable for the defocusing filter. This condition establishes that $\rho_M=1/P =AM/(N\lambda d_i)$,  which in turn results in $A=N\lambda d_i/(MP)$. Consequently, The first and second arguments simplify to $ k/N$ and $NC/(8MP)$, respectively. Among theses parameters for both arguments, only $C$ is related to the focused depth, as described by  (\ref{Eq4}). When the range of $d_f$ at each focused depth $d_f$, which is $(-\eta d_f,\eta d_f)$,  is discretized by $N_d$ equally spaced points, the corresponding depth values are defined by $\Delta d_{ft}=\frac{2t-N_d+1}{N_d-1}d_f$ for $t=0,1,…,N_d-1$. The parameter $C$ at the $t-$th position of the depth, denoted by $C_t$, is related to the selected parameter $C_{\!m\!a\!x}$ through (\ref{Eq4}) and (\ref{Eq18}) by (\ref{Eq22}).
\begin{equation}
\label{Eq22}
C_t= \frac{\Delta d_{ft}  C_{\!m\!a\!x}}{d_f+\Delta d_{ft}}\frac{1-\eta}{\eta}=  \frac{(2t-N_d+1) }{2t}\frac{1-\eta}{\eta}C_{\!m\!a\!x}
\end{equation}
Thus, $C$ and the second argument become independent of the focused depth $d_f$. When $C_{\!m\!a\!x}$ is chosen as an independent parameter, $f$ becomes $d_f$-dependent as given by (\ref{Eq20}), and the second argument, which shapes the defocusing filter, remains independent of $d_f$. Table~\ref{MainTable} in  Appendix \ref{AppendixDEF} presents the results of the total independent triples $(f_n ,C_{\!m\!a\!x},P)$ in $5\times7\times5=175$ records for the five-element array $(f_n,C_{\!m\!a\!x},P, \sigma_{\!m\!a\!x}, \!M\!A\!E_{\!m\!a\!x})$.

The implementation of the following filters highlights the potential of the results for practical applications.
\begin{itemize}
\item{For a specific threshold value $\!M\!A\!E_{th}$, the results filtered by the condition $\!M\!A\!E_{\!m\!a\!x}\leq \!M\!A\!E_{th}$ can be considered to satisfy the Gaussian model for the defocusing operator. As illustrated in Fig.\ref{Fig7}, a threshold value of $\!M\!A\!E_{th}=0.01$ represents an adequate acceptance level, indicating that the majority of the data points align with the Gaussian model. Decreasing the threshold value to achieve a more precise model fit does not necessarily ensure better estimation of the model parameter $\sigma$.}

\item{For a specific range $(\sigma_l,\sigma_u)$, the results filtered by the condition $\sigma_l<\sigma_{\!m\!a\!x}<\sigma_u$ determine the resolution power and the area size of local computing used by the method for extracting $\sigma$ at every image point. The computing area should be heuristically wider than the blur circle to ensure the reliability of estimating $\sigma$. A reasonable balance between the depth information captured within the local computing areas and the resolution of extracting $\sigma$ across the entire image can be achieved by setting $\sigma_l = P$ and $\sigma_u  = 5P$.}

\item{For a given threshold value $P_{th}$ for pixel size $P$, the results filtered by the condition $P\leq P_{th}$ support the popularity of applications.  Setting $P_{th}=5.6 \mu m$ is suitable for accommodating the application range of the majority of well-known conventional and smartphone cameras.}

\item{For a specific threshold value $f_{th}$ for the focal length $f$, the results filtered by the condition $f\leq f_{th}$ define the minimum angle of view of the camera.  Focal lengths exceeding $85 mm$ are commonly employed in telephoto cameras to facilitate the zooming in on and magnification of distant objects. Consequently, setting $f_{th}=100 mm$ encompasses the full range of wide-angle and normal imaging devices, as well as the majority of telephoto devices.}
\end{itemize}
\begin{figure}[!t]
\centering
\includegraphics[width=2.5in]{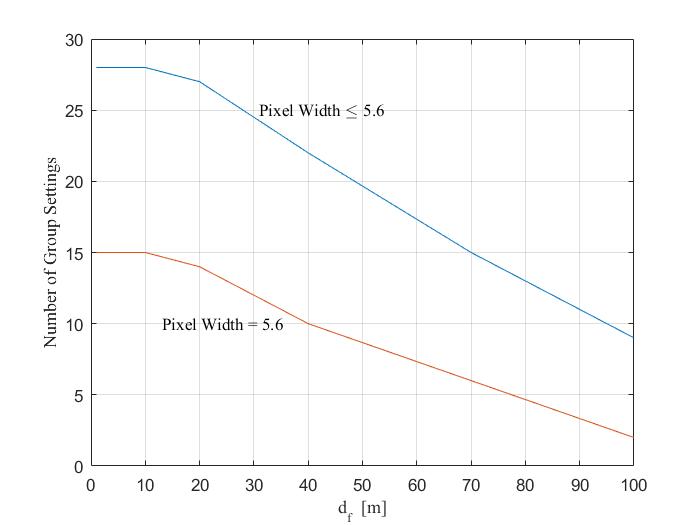}
\caption{Number of group settings  of an imaging device at focusing distances $\{1,5,10,20,40,70,100\} m$ for $P=5.6 \mu m$ and  $P\leq 5.6 \mu m$. Joined plot lines evident decreasing the number of options by increasing the focused distance. The settings and consequences $(f,f_n,\sigma_{\!m\!a\!x}, \!M\!A\!E_{\!m\!a\!x})$  are detailed in Table~\ref{GeneralTable} in the Appendix  \ref{AppendixDEF}.}
\label{Fig9}
\end{figure}

\begin{table*}
\centering 
\caption{The camera settings and resulted $\sigma_{\!m\!a\!x}$ and $\!M\!A\!E_{\!m\!a\!x}$ for depth finding at the focused depth $d_f\in\{1,5,10,20,40.70,100\}$ meters with the depth range $\pm10\%$  of  $d_f$ and pixel width $P=5.6\mu m.$\label{ParticularTable}}
\begin{tabular}{|*{3}{ccccc|}} 
\toprule
$d_{f}[m]$ & $f_n$ & $f[mm]$ & $\sigma_{\!m\!a\!x}[P]$ & $\!M\!A\!E_{\!m\!a\!x}$ & $d_f$ & $f_n$ & $f$ & $\sigma_{\!m\!a\!x}$ & $\!M\!A\!E_{\!m\!a\!x}$  & $d_f$ & $f_n$ & $f$ & $\sigma_{\!m\!a\!x}$ & $\!M\!A\!E_{\!m\!a\!x}$   \\
\midrule 
1 & 1.0 & 07.07 & 1.49 & 0.001 & 05 & 2.8 & 45.80 & 4.33 & 0.006 & 20 & 2.0 & 44.85 & 1.34 & 0.0022   \\ 
1 & 1.0 & 09.99 & 3.13 & 0.007 & 05 & 4.0 & 31.65 & 1.11 & 0.004 & 20 & 2.0 & 63.40 & 3.03 & 0.0029   \\ 
1 & 1.0 & 12.22 & 4.49 & 0.010 & 05 & 4.0 & 44.70 & 2.48 & 0.006 & 20 & 2.0 & 77.62 & 4.40 & 0.0084   \\ 
1 & 1.4 & 08.36 & 1.43 & 0.001 & 05 & 4.0 & 54.69 & 4.21 & 0.003 & 20 & 2.8 & 53.06 & 1.24 & 0.0031   \\ 
1 & 1.4 & 11.81 & 3.10 & 0.005 & 10 & 1.0 & 22.42 & 1.49 & 0.001 & 20 & 2.8 & 74.99 & 2.85 & 0.0030   \\ 
1 & 1.4 & 14.44 & 4.46 & 0.009 & 10 & 1.0 & 31.70 & 3.13 & 0.007 & 20 & 2.8 & 91.81 & 4.33 & 0.0065   \\ 
1 & 2.0 & 09.99 & 1.34 & 0.002 & 10 & 1.0 & 38.81 & 4.49 & 0.010 & 20 & 4.0 & 63.40 & 1.11 & 0.0038   \\ 
1 & 2.0 & 14.10 & 3.03 & 0.003 & 10 & 1.4 & 26.53 & 1.43 & 0.001 & 20 & 4.0 & 89.60 & 2.49 & 0.0056   \\ 
1 & 2.0 & 17.24 & 4.40 & 0.008 & 10 & 1.4 & 37.50 & 3.10 & 0.005 & 40 & 1.0 & 44.87 & 1.49 & 0.0006   \\ 
1 & 2.8 & 11.81 & 1.23 & 0.003 & 10 & 1.4 & 45.90 & 4.46 & 0.009 & 40 & 1.0 & 63.45 & 3.13 & 0.0069   \\ 
1 & 2.8 & 16.66 & 2.84 & 0.003 & 10 & 2.0 & 31.70 & 1.34 & 0.002 & 40 & 1.0 & 77.69 & 4.49 & 0.0097   \\ 
1 & 2.8 & 20.37 & 4.33 & 0.006 & 10 & 2.0 & 44.80 & 3.03 & 0.003 & 40 & 1.4 & 53.09 & 1.43 & 0.0012   \\ 
1 & 4.0 & 14.10 & 1.11 & 0.004 & 10 & 2.0 & 54.84 & 4.40 & 0.008 & 40 & 1.4 & 75.06 & 3.10 & 0.0052   \\ 
1 & 4.0 & 19.88 & 2.47 & 0.006 & 10 & 2.8 & 37.50 & 1.24 & 0.003 & 40 & 1.4 & 91.91 & 4.46 & 0.0095   \\ 
1 & 4.0 & 24.29 & 4.20 & 0.003 & 10 & 2.8 & 52.99 & 2.85 & 0.003 & 40 & 2.0 & 63.45 & 1.34 & 0.0022   \\ 
5 & 1.0 & 15.85 & 1.49 & 0.001 & 10 & 2.8 & 64.85 & 4.33 & 0.006 & 40 & 2.0 & 89.70 & 3.03 & 0.0029   \\ 
5 & 1.0 & 22.40 & 3.13 & 0.007 & 10 & 4.0 & 44.80 & 1.11 & 0.004 & 40 & 2.8 & 75.06 & 1.24 & 0.0031   \\ 
5 & 1.0 & 27.42 & 4.49 & 0.010 & 10 & 4.0 & 63.30 & 2.49 & 0.006 & 40 & 4.0 & 89.70 & 1.11 & 0.0038   \\ 
5 & 1.4 & 18.75 & 1.43 & 0.001 & 10 & 4.0 & 77.47 & 4.21 & 0.003 & 70 & 1.0 & 59.37 & 1.49 & 0.0006   \\ 
5 & 1.4 & 26.49 & 3.10 & 0.005 & 20 & 1.0 & 31.72 & 1.49 & 0.001 & 70 & 1.0 & 83.95 & 3.13 & 0.0069   \\ 
5 & 1.4 & 32.43 & 4.46 & 0.009 & 20 & 1.0 & 44.85 & 3.13 & 0.007 & 70 & 1.4 & 70.24 & 1.43 & 0.0012   \\ 
5 & 2.0 & 22.40 & 1.34 & 0.002 & 20 & 1.0 & 54.92 & 4.49 & 0.010 & 70 & 1.4 & 99.32 & 3.10 & 0.0052   \\ 
5 & 2.0 & 31.65 & 3.03 & 0.003 & 20 & 1.4 & 37.53 & 1.43 & 0.001 & 70 & 2.0 & 83.95 & 1.34 & 0.0022   \\ 
5 & 2.0 & 38.73 & 4.40 & 0.008 & 20 & 1.4 & 53.06 & 3.10 & 0.005 & 70 & 2.8 & 99.32 & 1.24 & 0.0031   \\ 
5 & 2.8 & 26.49 & 1.24 & 0.003 & 20 & 1.4 & 64.96 & 4.46 & 0.009 & 100 & 1.0 & 70.97 & 1.49 & 0.0006   \\ 
5 & 2.8 & 37.43 & 2.84 & 0.003 & 20 & 2.0 & 44.85 & 1.34 & 0.002 & 100 & 1.4 & 83.96 & 1.43 & 0.0012   \\ \hline
\end{tabular}
\end{table*}
Table~\ref{GeneralTable} in  Appendix  \ref{AppendixDEF} presents the results of the investigation into the practical applications of an imaging device for depth finding, filtered to exclude irrelevant data. The analysis identified $157$ records of group settings that align well with the Gaussian model for the defocusing operator. This group includes settings for both conventional and smartphone cameras. Table~\ref{ParticularTable} provides the detailed account of $77$ records of group settings that are specific to conventional cameras with a pixel size of $P=5.6 \mu m $ and focusing distances $d_f$ within the set ${1,5,10,20,40,70,100}$.  Fig.\ref{Fig9} illustrates the relationship between the number of group settings and focusing distance in cases where $P\leq5.6 \mu m$ and $P=5.6 \mu m$.

\section{Discussion and Conclusions}
\noindent This paper supported geometric optics by  diffraction limited theory to investigate the subject   “Defocus Aberration Theory Confirms Gaussian Model for Defocus Operator in Most Imaging Devices”. The theory made a linear relationship between the size of the CoC in geometrical optics and the maximum phase shift $A_R$ for the wave front across the pupil caused by defocus. This link opened an investigation channel into defocused imaging systems for characterizing their Monochrome OTF. By modeling the ambient illumination as black body radiation, it became possible to calculate the OTF under natural lighting conditions. The practical frequency band of the magnitude of the OTF was fitted to the Gaussian model over $1225$ records from conventional settings and applications within the following ranges: focused distance from $1$ to $100 m$, depth range from $-10\%$ to $+10\%$  of the focused distance,  focal number from $1$ to $4$, pixel width ($P$) from $1$ to  $8 \mu m$,  maximum blur circle diameter of $1$ to $7$ pixels, and image intensity frequency under $1$ cycle per pixel.This range accommodates the near by future conventional cameras that will support super-resolution hardware with double the original resolution.
\begin{table}[!ht]
    \centering
    \caption{The statistics of results in the Table~\ref{GeneralTable} in Appendix  \ref{AppendixDEF} for number of options and the extremes of pixel width, focal length and f-number at each focused depth. \label{GeneralTableStatistics}}
    \begin{tabular}{ccccc}
        $d_f [m]$  & No. of     & $P [\mu m]$ &          $f [mm] $&   $f_n$ \\
     Value& Options & Min \quad Mux& Min  \quad Mux & Min \quad Mux  \\ \toprule
        1    & 28 & 2 \quad 5.6  & 4.23  \quad     24.29 & 1 \quad       4.0 \\ 
        5    & 28 & 2 \quad 5.6  & 9.48  \quad     54.69 & 1 \quad       4.0 \\ 
        10  & 28 & 2  \quad 5.6 & 13.41 \quad    77.47 & 1 \quad       4.0 \\ 
        20  & 27 & 2 \quad  5.6 & 18.97  \quad   91.81 & 1 \quad       4.0 \\ 
        40  & 22 & 2 \quad 5.6  & 26.82  \quad  92.84  & 1 \quad       4.0\\ 
        70  & 15 & 2 \quad 5.6  & 35.49  \quad   99.32 & 1 \quad       2.8 \\ 
        100 & 9  & 2 \quad 5.6  & 42.42  \quad   86.91 & 1 \quad       2.0 \\ 
    \end{tabular}
\end{table}
Each record was also characterised by the resulting focal length $f$, the maximum mean absolute error of fitting $\!M\!A\!E_{\!m\!a\!x}$, and the maximum estimated parameter for the model $\sigma_{\!m\!a\!x}$. Applying the filters $\!M\!A\!E_{\!m\!a\!x}\leq 0.01,  1\leq P \leq 5.6 \mu m , f\leq 100 mm$ and $1\leq \sigma_{\!m\!a\!x}\leq 5$ identified $157$ records of the group settings that align well with the Gaussian model for the defocusing operator and support local computations to extract the model parameter with acceptable resolution. The statistics of records for each focused depth are given in the Table~\ref{GeneralTableStatistics}. It was demonstrated that the number of options increases as the focused depth decreases. Re-filtering the results to include only those with the pixel width of $5.6 \mu m$ yielded $77$ records within the group settings. This refined dataset supports a wide range of applications in the literature for enhancing the efficiency of 3-D recovery. The supporting range for conventional camera settings without super-resolution power, either results in a $\!M\!A\!E_{\!m\!a\!x}$ significantly less than $1\%$, or increases with the number of records while maintaining the same error rate. This enhances the reliability of the Gaussian model for the defocusing operator and its standard deviation as a depth measure in conventional imaging systems.

The Gaussian model is almost perfect for defocusing operators in conventional imaging devices, up to sampling by the  image sensors grid. Despite the potential for aliasing effects in captured images, scene images adhering to the Nyquist rate ($fp/2$) remain within the confidential range of applications.   
The model was studied under diffraction-limited conditions with defocus aberration. All lens aberrations, including those associated with monochromatic light (spherical aberration, coma, astigmatism, field curvature, image distortion) and chromatic aberrations (dispersion), were not considered in the research. Conventional cameras benefit from the increasing demand by microscopy and telescopy imaging systems for aberration correction techniques to improve image quality. This places conventional cameras within an acceptable range of lens aberrations for image processing. 
The research is ongoing on the computational structure for estimating the standard deviation of the defocusing Gaussian filter at an image point from a sharper version of the image. The high confidence in the Gaussian filter makes it an attractive option for software and hardware implementation, aimed at accelerating extensive image processing applications towards real-time capabilities. The next phase of the research will focus on outlining the fabrication module for this realization.

\begin{IEEEbiography}[{\includegraphics[width=1in,height=1.25in,clip,keepaspectratio]{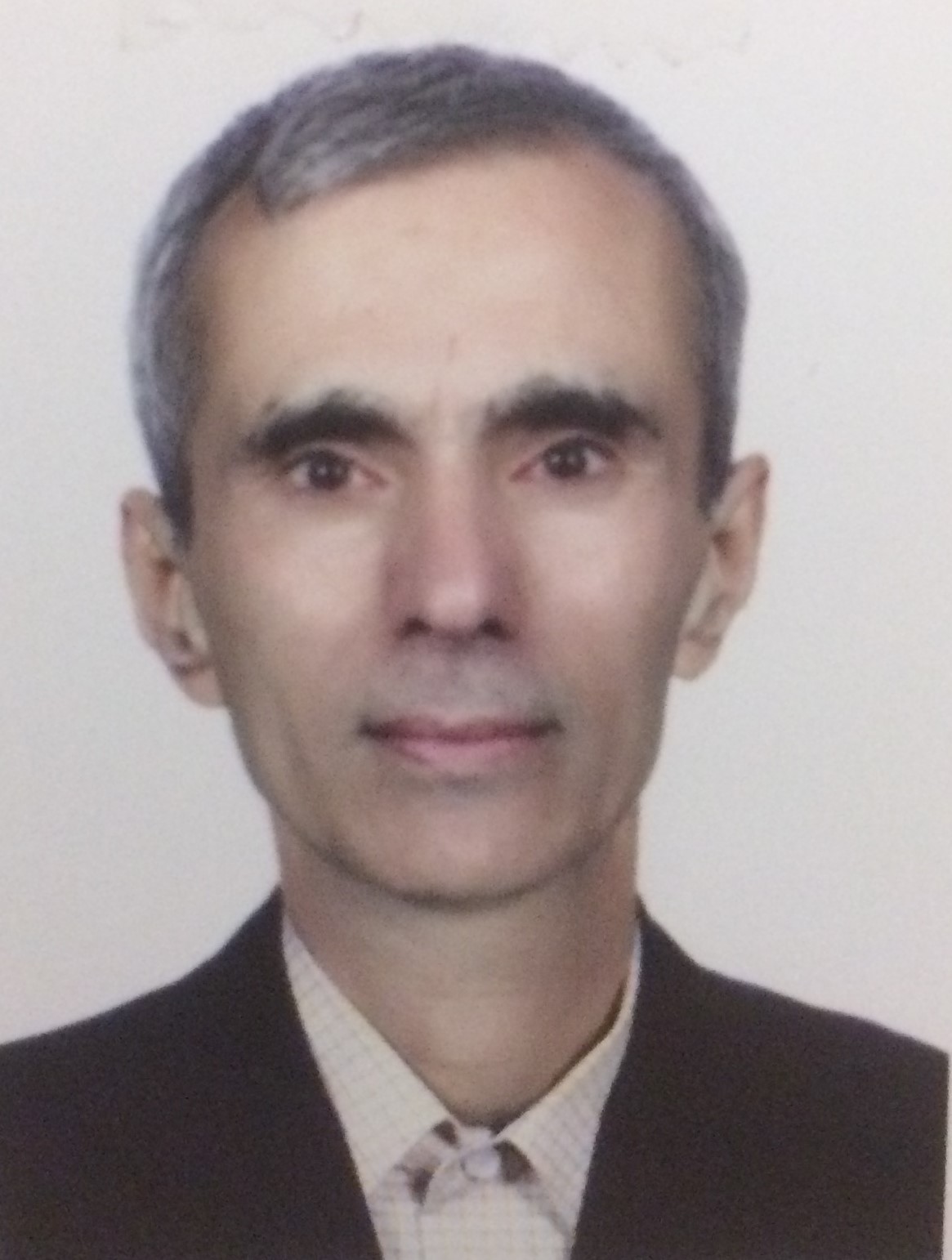}}]
Akbar Saadat  received the B.S. and M.S. degrees from Esfahan University of Technology, in 1986 and 1991, and the Ph.D. degree from Sharif University of Technology, in 1997, all in electrical engineering, Iran. He has been faculty member of Electrical Engineering Department in Esfahan University of Technology and Yazd University. Since 1998, he has been with the R\&D, Signalling, and Infrastructure Technical Departments at Iranian Railways. He is an expert in railway signalling and has completed several international training courses in Germany, India, and Japan. His research interests include computer vision, image analyzis, and information processing..
\end{IEEEbiography}

\appendices
\section{ justifying the equation (\ref{Eq10})}
\label{AppendixEq10}
The expression for $H_{def}^{o} (\rho)$ is derived from $H_{def} (f_x,f_y )$ as follows.
\begin{equation}
\label{EqA1}
H_{def}^{o} (\rho) = H_{def} (\sqrt{f_x^2,f_y^2} )|_{f_x=\rho,f_y=0}=H_{def}(\rho,0)
\end{equation}
The auxiliary function $\psi (P(x,y),\lambda d_i \rho)$ is defined as (\ref{EqA2}) to simplify the integrands in  (\ref{Eq10}). Taking into account $P_{def} (x,y)=P(x,y)exp(j\kappa W(x,y))$  and  $W(x,y) = \frac{A_R}{R^2}(x^2+y^2)$ yields.
\begin{align}
&\Psi(P(x,y),\lambda d_i \rho)  \stackrel{\triangle}{=} \label{EqA2} \\
&P_{def}(x+\frac{\lambda d_i \rho}{2},y)  P_{def}(x-\frac{\lambda d_i \rho}{2},y) = \nonumber \\
&P(x+\frac{\lambda d_i \rho}{2},y)  P(x-\frac{\lambda d_i \rho}{2},y) \exp(j8\pi\frac{A_R}{\lambda}\frac{x\lambda d_i\rho}{2R^2})  \nonumber
\end{align}
$H_{def}^{o} (\rho)$ is obtained by (\ref{EqA3}) in using $\Psi$ as an integrand.
\begin{equation}
H_{def}^{o} (\rho)  =H_{def}(\rho,0)  \label{EqA3}=\frac{\iint\limits_{-\infty}^{+\infty}\Psi(P(x,y),\lambda d_i \rho)dxdy}{\iint\limits_{-\infty}^{+\infty}\Psi(P(x,y),0)dxdy} \\
\end{equation}
For the circular aperture with the radius R the denominator is equal to the aperture area $\pi R^2$. The integrand in the nominator has the exponent part value in the overlap area of the circles centered at $(x,y)= (-\frac{\lambda d_i \rho}{2},0)$ and $(\frac{\lambda d_i \rho}{2},0)$. The overlap area is formed when the distance between the center points  is less than the diameter of the circles that is  $\lambda d_i \rho<2R$ or $ \rho<\frac{2R}{\lambda d_i}=2\rho_o$. This area is characterised by the range $(|x|\leq x_m, |y|\leq y_m(x))$  where $x_m=R-\frac{\lambda d_i \rho}{2}$ and $y_m(x)=\sqrt{R^2-(|x|+\frac{\lambda d_i \rho}{2})^2}$. Therefore, $H_{def}^{o} (\rho)$ is obtained as  (\ref{EqA4}). 
\begin{align}
H_{def}^{o}&=\frac{1}{\pi R^2} \iint\limits_{-\infty}^{+\infty}\Psi(P(x,y),\lambda d_i \rho)dxdy   \label{EqA4} \\
&=\frac{1}{\pi R^2}\int_{-x_m}^{x_m}2y_m(x)\exp(j8\pi\frac{A_R}{\lambda}\frac{x\lambda d_i\rho}{2R^2})dx  \nonumber \\
&=\frac{1}{\pi R^2}\int_{0}^{x_m}4y_m(x)\cos(8\pi\frac{A_R}{\lambda}\frac{x\lambda d_i\rho}{2R^2})dx  \nonumber 
\end{align}
Replacing $ \rho$ with $2\rho_o\cos(\theta)$ and normalising integration variable to $R$, shapes the result the same as  (\ref{Eq10}).

\section{The  unique $k(\theta)$ }
\label{AppendixEq11}
The original function in  (\ref{EqB1}) describes the unit circle centered at $(-\cos(\theta),0)$ in the first quadrant of the $xy$ plane.  
\begin{equation}
 \label{EqB1} 
y = \sqrt{1-(x+\cos(\theta))^2}  \approx\sin(\theta)(1-(\frac{x}{1-\cos(\theta)})^{k})
\end{equation}
All approximations match with the exact values of the original function at  $x=0$ and $x=1-\cos(\theta)$. The value of original function at the mid point on that circle is $y=\sin(\theta/2)$ at $x=\cos(\theta/2)-\cos(\theta)$.  The value of $k(\theta)$ is chosen by  (\ref{EqB2}) to force the approximation to be satisfied by the mid point.
\begin{equation}
 \label{EqB2} 
 \sin(\frac{\theta}{2}) = \sin(\theta)(1-(\frac{\cos(\theta/2)-\cos(\theta)}{1-\cos(\theta)})^{k}) 
\end{equation}
The solution of  (\ref{EqB2}) for  $k(\theta)$ and the mean value of that over the  range $0<\theta<\pi/2$ is given by  (\ref{EqB3}).
\begin{align}
 \label{EqB3} 
 &k(\theta) = \frac{\log\left[\frac{\sin(\theta)- \sin(\theta/2)}{\sin(\theta)}\right]}{\log\left[\frac{\cos(\theta/2)-\cos(\theta)}{1-\cos(\theta)}\right]} \\ 
&\bar{k} = \frac{2}{\pi}\int_{0}^{\frac{\pi}{2}}k(\theta)d\theta = 2.70428. \nonumber
\end{align}

\section{$ \!M\!A\!E$ is bounded by $\!R\!M\!S\!E$}
\label{AppendixDEF}
Considering the mean of a random variable $X$  as $\overline{X}$ and its error as $X- \overline{X}$. When $\overline{X}=0$ the Root Mean Square Error ($\!R\!M\!S\!E$) and the Mean Absolute Error ($ \!M\!A\!E$) of  $X$ are simplified as  (\ref{EqC1}).
\begin{align}
 \label{EqC1} 
&\!R\!M\!S\!E(X) = \sqrt{\overline{(X- \overline{X})^2}}=  \sqrt{\overline{X^2}}\\
& \!M\!A\!E(X) = \overline{|X-\overline{X}|}= \overline{\sqrt{(X-\overline{X})^2}}= \overline{\sqrt{X^2}} \nonumber
\end{align}
For any n outcome of the positive random variable $X^2$ such as $ x_1,x_2,x_3,\dots, x_n$, an estimation of  $\!R\!M\!S\!E(X)$ and $ \!M\!A\!E(X)$ are given by $R_n $ and $M_n$ by  (\ref{EqC2}).
\begin{align}
 \label{EqC2} 
&R_n = \sqrt{\frac{x_1+x_2+x_3+\dots + x_n}{n}}\\
&M_n = \frac{\sqrt{x_1}+\sqrt{x_2}+\sqrt{x_3}+\dots+\sqrt{x_n}}{n}\nonumber
\end{align}
The following theorem states that $ \!M\!A\!E$ is bounded by $\!R\!M\!S\!E$.
\begin{theorem}
for any positive integer $n$ and any set of positive numbers $\{ x_1,x_2,x_3,\dots, x_n \}:\quad M_n\leq R_n$
\end{theorem}
\begin{proof}
Using the mathematical induction,  since $M_1= \frac{\sqrt{x_1}}{1}=\sqrt{\frac{x_1}{1}} =R_1$, the statement is true for $n=1$.
Assume $M_k\leq R_k$ for some positive integer $n=k$. The statement for  $n=k+1$ is also true as demonstrated by (\ref{EqC3}).
\begin{align}
 \label{EqC3} 
&M_{k+1}-R_{k+1} =\frac{M_{k+1}^2-R_{k+1}^2}{M_{k+1}+R_{k+1}}\\
&=\frac{(\frac{kM_k+\sqrt{x_{k+1}}}{k+1})^2-\frac{kR_k^2+x_{k+1}}{k+1}}{M_{k+1}+R_{k+1}} \nonumber \\ 
&=\frac{-k \left((M_k-\sqrt{x_{k+1}})^2+(k+1)(R_k^2-M_k^2)\right)}{(k+1)^2(M_{k+1}+(R_{k+1})}\leq 0 \nonumber
\end{align}   
\end{proof}

\begin{table*}
\centering 
\caption{The camera settings and resulted $\sigma_{\!m\!a\!x}$ and $\!M\!A\!E_{\!m\!a\!x}$ for depth finding at the focused depth $d_f\in\{1,5,10,20,40.70,100\}$ meters with the depth range $\pm10\%$  of  $d_f$. The focal length $f$ for each record and  separately for each focused depth $d_f$ is given by (\ref{Eq20}). \label{MainTable}}
\begin{tabular}{|*{3}{ccccc|}}     
\toprule
$f_n$ & $C_{\!m\!a\!x}[P]$ & $P[\mu m]$  & $\sigma_{\!m\!a\!x}[P]$ & $\!M\!A\!E_{\!m\!a\!x}$ & $f_n$ & $C_{\!m\!a\!x}$ & $P$  & $\sigma_{\!m\!a\!x}$ & $\!M\!A\!E_{\!m\!a\!x}$ & $f_n$ & $C_{\!m\!a\!x}$ & $P$  & $\sigma_{\!m\!a\!x}$ & $\!M\!A\!E_{\!m\!a\!x}$  \\
\midrule 
1.0 & 1 & 1.0 & 0.99 & 0.0232 & 1.4 & 5 & 5.6 & 7.13 & 0.0094 & 2.8 & 3 & 2.0 & 3.03 & 0.0340   \\ 
1.0 & 1 & 2.0 & 1.24 & 0.0086 & 1.4 & 5 & 8.0 & 7.16 & 0.0066 & 2.8 & 3 & 4.0 & 4.22 & 0.0041   \\ 
1.0 & 1 & 4.0 & 1.43 & 0.0017 & 1.4 & 6 & 1.0 & 7.83 & 0.0301 & 2.8 & 3 & 5.6 & 4.33 & 0.0064   \\ 
1.0 & 1 & 5.6 & 1.49 & 0.0006 & 1.4 & 6 & 2.0 & 8.41 & 0.0199 & 2.8 & 3 & 8.0 & 4.40 & 0.0059   \\ 
1.0 & 1 & 8.0 & 1.55 & 0.0004 & 1.4 & 6 & 4.0 & 8.47 & 0.0115 & 2.8 & 4 & 1.0 & 2.69 & 0.1248   \\ 
1.0 & 2 & 1.0 & 2.15 & 0.0466 & 1.4 & 6 & 5.6 & 8.46 & 0.0087 & 2.8 & 4 & 2.0 & 4.68 & 0.0292   \\ 
1.0 & 2 & 2.0 & 2.85 & 0.0086 & 1.4 & 6 & 8.0 & 8.47 & 0.0063 & 2.8 & 4 & 4.0 & 5.64 & 0.0084   \\ 
1.0 & 2 & 4.0 & 3.10 & 0.0073 & 1.4 & 7 & 1.0 & 9.18 & 0.0315 & 2.8 & 4 & 5.6 & 5.66 & 0.0082   \\ 
1.0 & 2 & 5.6 & 3.13 & 0.0069 & 1.4 & 7 & 2.0 & 9.79 & 0.0185 & 2.8 & 4 & 8.0 & 5.73 & 0.0063   \\ 
1.0 & 2 & 8.0 & 3.15 & 0.0056 & 1.4 & 7 & 4.0 & 9.83 & 0.0105 & 2.8 & 5 & 1.0 & 3.67 & 0.1255   \\ 
1.0 & 3 & 1.0 & 3.77 & 0.0420 & 1.4 & 7 & 5.6 & 9.81 & 0.0079 & 2.8 & 5 & 2.0 & 6.43 & 0.0170   \\ 
1.0 & 3 & 2.0 & 4.33 & 0.0182 & 1.4 & 7 & 8.0 & 9.80 & 0.0058 & 2.8 & 5 & 4.0 & 6.94 & 0.0099   \\ 
1.0 & 3 & 4.0 & 4.46 & 0.0132 & 2.0 & 1 & 1.0 & 0.72 & 0.0216 & 2.8 & 5 & 5.6 & 7.08 & 0.0082   \\ 
1.0 & 3 & 5.6 & 4.49 & 0.0097 & 2.0 & 1 & 2.0 & 0.99 & 0.0116 & 2.8 & 5 & 8.0 & 7.11 & 0.0063   \\ 
1.0 & 3 & 8.0 & 4.51 & 0.0070 & 2.0 & 1 & 4.0 & 1.23 & 0.0043 & 2.8 & 6 & 1.0 & 4.88 & 0.1228   \\ 
1.0 & 4 & 1.0 & 5.49 & 0.0160 & 2.0 & 1 & 5.6 & 1.34 & 0.0022 & 2.8 & 6 & 2.0 & 7.81 & 0.0150   \\ 
1.0 & 4 & 2.0 & 5.66 & 0.0231 & 2.0 & 1 & 8.0 & 1.43 & 0.0009 & 2.8 & 6 & 4.0 & 8.40 & 0.0099   \\ 
1.0 & 4 & 4.0 & 5.78 & 0.0138 & 2.0 & 2 & 1.0 & 1.49 & 0.0662 & 2.8 & 6 & 5.6 & 8.47 & 0.0077   \\ 
1.0 & 4 & 5.6 & 5.82 & 0.0102 & 2.0 & 2 & 2.0 & 2.14 & 0.0234 & 2.8 & 6 & 8.0 & 8.47 & 0.0057   \\ 
1.0 & 4 & 8.0 & 5.85 & 0.0073 & 2.0 & 2 & 4.0 & 2.84 & 0.0043 & 2.8 & 7 & 1.0 & 6.24 & 0.1101   \\ 
1.0 & 5 & 1.0 & 6.82 & 0.0313 & 2.0 & 2 & 5.6 & 3.03 & 0.0029 & 2.8 & 7 & 2.0 & 9.17 & 0.0158   \\ 
1.0 & 5 & 2.0 & 7.09 & 0.0232 & 2.0 & 2 & 8.0 & 3.10 & 0.0036 & 2.8 & 7 & 4.0 & 9.79 & 0.0093   \\ 
1.0 & 5 & 4.0 & 7.13 & 0.0132 & 2.0 & 3 & 1.0 & 2.39 & 0.0942 & 2.8 & 7 & 5.6 & 9.82 & 0.0069   \\ 
1.0 & 5 & 5.6 & 7.15 & 0.0095 & 2.0 & 3 & 2.0 & 3.76 & 0.0212 & 2.8 & 7 & 8.0 & 9.83 & 0.0052   \\ 
1.0 & 5 & 8.0 & 7.17 & 0.0067 & 2.0 & 3 & 4.0 & 4.33 & 0.0090 & 4.0 & 1 & 1.0 & 0.47 & 0.0139   \\ 
1.0 & 6 & 1.0 & 8.17 & 0.0336 & 2.0 & 3 & 5.6 & 4.40 & 0.0084 & 4.0 & 1 & 2.0 & 0.72 & 0.0108   \\ 
1.0 & 6 & 2.0 & 8.47 & 0.0215 & 2.0 & 3 & 8.0 & 4.46 & 0.0066 & 4.0 & 1 & 4.0 & 0.99 & 0.0058   \\ 
1.0 & 6 & 4.0 & 8.46 & 0.0122 & 2.0 & 4 & 1.0 & 3.55 & 0.0951 & 4.0 & 1 & 5.6 & 1.11 & 0.0038   \\ 
1.0 & 6 & 5.6 & 8.47 & 0.0090 & 2.0 & 4 & 2.0 & 5.48 & 0.0078 & 4.0 & 1 & 8.0 & 1.23 & 0.0022   \\ 
1.0 & 6 & 8.0 & 8.48 & 0.0063 & 2.0 & 4 & 4.0 & 5.66 & 0.0115 & 4.0 & 2 & 1.0 & 0.96 & 0.0500   \\ 
1.0 & 7 & 1.0 & 9.55 & 0.0336 & 2.0 & 4 & 5.6 & 5.72 & 0.0090 & 4.0 & 2 & 2.0 & 1.48 & 0.0331   \\ 
1.0 & 7 & 2.0 & 9.82 & 0.0194 & 2.0 & 4 & 8.0 & 5.78 & 0.0069 & 4.0 & 2 & 4.0 & 2.13 & 0.0118   \\ 
1.0 & 7 & 4.0 & 9.81 & 0.0111 & 2.0 & 5 & 1.0 & 5.02 & 0.0877 & 4.0 & 2 & 5.6 & 2.47 & 0.0057   \\ 
1.0 & 7 & 5.6 & 9.80 & 0.0082 & 2.0 & 5 & 2.0 & 6.81 & 0.0156 & 4.0 & 2 & 8.0 & 2.83 & 0.0022   \\ 
1.0 & 7 & 8.0 & 9.80 & 0.0059 & 2.0 & 5 & 4.0 & 7.08 & 0.0115 & 4.0 & 3 & 1.0 & 1.46 & 0.0947   \\ 
1.4 & 1 & 1.0 & 0.86 & 0.0235 & 2.0 & 5 & 5.6 & 7.11 & 0.0090 & 4.0 & 3 & 2.0 & 2.38 & 0.0473   \\ 
1.4 & 1 & 2.0 & 1.12 & 0.0104 & 2.0 & 5 & 8.0 & 7.13 & 0.0066 & 4.0 & 3 & 4.0 & 3.73 & 0.0108   \\ 
1.4 & 1 & 4.0 & 1.34 & 0.0030 & 2.0 & 6 & 1.0 & 6.65 & 0.0688 & 4.0 & 3 & 5.6 & 4.20 & 0.0026   \\ 
1.4 & 1 & 5.6 & 1.43 & 0.0012 & 2.0 & 6 & 2.0 & 8.16 & 0.0167 & 4.0 & 3 & 8.0 & 4.33 & 0.0044   \\ 
1.4 & 1 & 8.0 & 1.50 & 0.0004 & 2.0 & 6 & 4.0 & 8.47 & 0.0107 & 4.0 & 4 & 1.0 & 2.01 & 0.1331   \\ 
1.4 & 2 & 1.0 & 1.82 & 0.0613 & 2.0 & 6 & 5.6 & 8.47 & 0.0082 & 4.0 & 4 & 2.0 & 3.53 & 0.0480   \\ 
1.4 & 2 & 2.0 & 2.51 & 0.0154 & 2.0 & 6 & 8.0 & 8.46 & 0.0061 & 4.0 & 4 & 4.0 & 5.47 & 0.0041   \\ 
1.4 & 2 & 4.0 & 3.04 & 0.0042 & 2.0 & 7 & 1.0 & 8.29 & 0.0523 & 4.0 & 4 & 5.6 & 5.63 & 0.0058   \\ 
1.4 & 2 & 5.6 & 3.10 & 0.0052 & 2.0 & 7 & 2.0 & 9.55 & 0.0167 & 4.0 & 4 & 8.0 & 5.66 & 0.0057   \\ 
1.4 & 2 & 8.0 & 3.13 & 0.0049 & 2.0 & 7 & 4.0 & 9.82 & 0.0096 & 4.0 & 5 & 1.0 & 2.64 & 0.1548   \\ 
1.4 & 3 & 1.0 & 3.05 & 0.0675 & 2.0 & 7 & 5.6 & 9.83 & 0.0075 & 4.0 & 5 & 2.0 & 4.98 & 0.0445   \\ 
1.4 & 3 & 2.0 & 4.22 & 0.0084 & 2.0 & 7 & 8.0 & 9.81 & 0.0055 & 4.0 & 5 & 4.0 & 6.81 & 0.0077   \\ 
1.4 & 3 & 4.0 & 4.40 & 0.0119 & 2.8 & 1 & 1.0 & 0.59 & 0.0182 & 4.0 & 5 & 5.6 & 6.93 & 0.0070   \\ 
1.4 & 3 & 5.6 & 4.46 & 0.0094 & 2.8 & 1 & 2.0 & 0.86 & 0.0117 & 4.0 & 5 & 8.0 & 7.08 & 0.0057   \\ 
1.4 & 3 & 8.0 & 4.50 & 0.0068 & 2.8 & 1 & 4.0 & 1.12 & 0.0052 & 4.0 & 6 & 1.0 & 3.37 & 0.1585   \\ 
1.4 & 4 & 1.0 & 4.70 & 0.0574 & 2.8 & 1 & 5.6 & 1.23 & 0.0031 & 4.0 & 6 & 2.0 & 6.59 & 0.0353   \\ 
1.4 & 4 & 2.0 & 5.64 & 0.0170 & 2.8 & 1 & 8.0 & 1.34 & 0.0015 & 4.0 & 6 & 4.0 & 8.15 & 0.0083   \\ 
1.4 & 4 & 4.0 & 5.73 & 0.0127 & 2.8 & 2 & 1.0 & 1.21 & 0.0614 & 4.0 & 6 & 5.6 & 8.38 & 0.0069   \\ 
1.4 & 4 & 5.6 & 5.78 & 0.0098 & 2.8 & 2 & 2.0 & 1.81 & 0.0307 & 4.0 & 6 & 8.0 & 8.47 & 0.0054   \\ 
1.4 & 4 & 8.0 & 5.82 & 0.0072 & 2.8 & 2 & 4.0 & 2.50 & 0.0078 & 4.0 & 7 & 1.0 & 4.24 & 0.1584   \\ 
1.4 & 5 & 1.0 & 6.44 & 0.0329 & 2.8 & 2 & 5.6 & 2.84 & 0.0031 & 4.0 & 7 & 2.0 & 8.24 & 0.0270   \\ 
1.4 & 5 & 2.0 & 6.95 & 0.0200 & 2.8 & 2 & 8.0 & 3.03 & 0.0021 & 4.0 & 7 & 4.0 & 9.53 & 0.0083   \\ 
1.4 & 5 & 4.0 & 7.11 & 0.0127 & 2.8 & 3 & 1.0 & 1.89 & 0.1042 & 4.0 & 7 & 5.6 & 9.78 & 0.0066   \\ 
1.4 & 5 & 5.6 & 7.13 & 0.0094 & 2.8 & 3 & 2.0 & 3.03 & 0.0340 & 4.0 & 7 & 8.0 & 9.82 & 0.0048   \\   \hline
\end{tabular}
\end{table*}

\begin{table*}
\centering 
\caption{The practical general camera settings and resulted $\sigma_{\!m\!a\!x}$ and $\!M\!A\!E_{\!m\!a\!x}$ given in the Table~\ref{MainTable} in the  Appendix \ref{AppendixDEF} filtered for the pixel width $P\leq5.6\mu m$, the focal length $f<100 mm$, maximum error $\!M\!A\!E_{\!m\!a\!x}<0.01$ and $\sigma_{\!m\!a\!x}<5$. \label{GeneralTable}}   
\begin{tabular}{|*{3}{>{\centering}p{0.01\textwidth}>{\centering}p{0.01\textwidth}>{\centering}p{0.02\textwidth}>{\centering}p{0.01\textwidth}>{\centering}p{0.03\textwidth}>{\centering}p{0.02\textwidth}p{0.05\textwidth}|}}        
\toprule
$\frac{d_f}{m}$ & $f_n$ & $\frac{C_{\!m\!a\!x}}{[P]}$ & $\frac{P}{[\mu m]}$  & $\frac{f}{[mm]}$  & $\frac{\sigma_{\!m\!a\!x}}{[P]}$ & $\!M\!A\!E_{\!m\!a\!x}$ & $d_f$ & $f_n$ & $C_{\!m\!a\!x}$ & $P$ & $f$  & $\sigma_{\!m\!a\!x}$ & $\!M\!A\!E_{\!m\!a\!x}$ & $d_f$ & $f_n$ & $C_{\!m\!a\!x}$ & $P$  & $f$ & $\sigma_{\!m\!a\!x}$ & $\!M\!A\!E_{\!m\!a\!x}$  \\
\midrule 
1 & 1.0 & 1 & 2.0 & 04.23 & 1.24 & 0.0086 & 05 & 2.8 & 3 & 5.6 & 45.80 & 4.33 & 0.0065 & 020 & 2.8 & 1 & 5.6 & 53.06 & 1.24 & 0.0031   \\ 
1 & 1.0 & 1 & 4.0 & 05.98 & 1.43 & 0.0017 & 05 & 4.0 & 1 & 5.6 & 31.65 & 1.11 & 0.0038 & 020 & 2.8 & 2 & 4.0 & 63.40 & 2.51 & 0.0077   \\ 
1 & 1.0 & 1 & 5.6 & 07.07 & 1.49 & 0.0006 & 05 & 4.0 & 2 & 5.6 & 44.70 & 2.48 & 0.0056 & 020 & 2.8 & 2 & 5.6 & 74.99 & 2.85 & 0.0030   \\ 
1 & 1.0 & 2 & 2.0 & 05.98 & 2.85 & 0.0086 & 05 & 4.0 & 3 & 5.6 & 54.69 & 4.21 & 0.0027 & 020 & 2.8 & 3 & 4.0 & 77.62 & 4.22 & 0.0043   \\ 
1 & 1.0 & 2 & 4.0 & 08.45 & 3.10 & 0.0073 & 10 & 1.0 & 1 & 2.0 & 13.41 & 1.24 & 0.0086 & 020 & 2.8 & 3 & 5.6 & 91.81 & 4.33 & 0.0065   \\ 
1 & 1.0 & 2 & 5.6 & 09.99 & 3.13 & 0.0069 & 10 & 1.0 & 1 & 4.0 & 18.96 & 1.43 & 0.0017 & 020 & 4.0 & 1 & 5.6 & 63.40 & 1.11 & 0.0038   \\ 
1 & 1.0 & 3 & 5.6 & 12.22 & 4.49 & 0.0097 & 10 & 1.0 & 1 & 5.6 & 22.42 & 1.49 & 0.0006 & 020 & 4.0 & 2 & 5.6 & 89.60 & 2.49 & 0.0056   \\ 
1 & 1.4 & 1 & 4.0 & 07.07 & 1.34 & 0.0030 & 10 & 1.0 & 2 & 2.0 & 18.96 & 2.85 & 0.0085 & 040 & 1.0 & 1 & 2.0 & 26.82 & 1.24 & 0.0086   \\ 
1 & 1.4 & 1 & 5.6 & 08.36 & 1.43 & 0.0012 & 10 & 1.0 & 2 & 4.0 & 26.80 & 3.10 & 0.0073 & 040 & 1.0 & 1 & 4.0 & 37.93 & 1.43 & 0.0017   \\ 
1 & 1.4 & 2 & 4.0 & 09.99 & 3.04 & 0.0042 & 10 & 1.0 & 2 & 5.6 & 31.70 & 3.13 & 0.0069 & 040 & 1.0 & 1 & 5.6 & 44.87 & 1.49 & 0.0006   \\ 
1 & 1.4 & 2 & 5.6 & 11.81 & 3.10 & 0.0052 & 10 & 1.0 & 3 & 5.6 & 38.81 & 4.49 & 0.0097 & 040 & 1.0 & 2 & 2.0 & 37.93 & 2.85 & 0.0084   \\ 
1 & 1.4 & 3 & 2.0 & 08.66 & 4.22 & 0.0084 & 10 & 1.4 & 1 & 4.0 & 22.42 & 1.34 & 0.0029 & 040 & 1.0 & 2 & 4.0 & 53.63 & 3.10 & 0.0073   \\ 
1 & 1.4 & 3 & 5.6 & 14.44 & 4.46 & 0.0094 & 10 & 1.4 & 1 & 5.6 & 26.53 & 1.43 & 0.0012 & 040 & 1.0 & 2 & 5.6 & 63.45 & 3.13 & 0.0069   \\ 
1 & 2.0 & 1 & 4.0 & 08.45 & 1.23 & 0.0043 & 10 & 1.4 & 2 & 4.0 & 31.70 & 3.04 & 0.0043 & 040 & 1.0 & 3 & 5.6 & 77.69 & 4.49 & 0.0097   \\ 
1 & 2.0 & 1 & 5.6 & 09.99 & 1.34 & 0.0022 & 10 & 1.4 & 2 & 5.6 & 37.50 & 3.10 & 0.0052 & 040 & 1.4 & 1 & 4.0 & 44.87 & 1.34 & 0.0029   \\ 
1 & 2.0 & 2 & 4.0 & 11.93 & 2.84 & 0.0043 & 10 & 1.4 & 3 & 2.0 & 27.46 & 4.23 & 0.0087 & 040 & 1.4 & 1 & 5.6 & 53.09 & 1.43 & 0.0012   \\ 
1 & 2.0 & 2 & 5.6 & 14.10 & 3.03 & 0.0029 & 10 & 1.4 & 3 & 5.6 & 45.90 & 4.46 & 0.0095 & 040 & 1.4 & 2 & 4.0 & 63.45 & 3.04 & 0.0043   \\ 
1 & 2.0 & 3 & 4.0 & 14.59 & 4.33 & 0.0090 & 10 & 2.0 & 1 & 4.0 & 26.80 & 1.24 & 0.0043 & 040 & 1.4 & 2 & 5.6 & 75.06 & 3.10 & 0.0052   \\ 
1 & 2.0 & 3 & 5.6 & 17.24 & 4.40 & 0.0084 & 10 & 2.0 & 1 & 5.6 & 31.70 & 1.34 & 0.0022 & 040 & 1.4 & 3 & 2.0 & 54.95 & 4.23 & 0.0087   \\ 
1 & 2.8 & 1 & 4.0 & 09.99 & 1.12 & 0.0052 & 10 & 2.0 & 2 & 4.0 & 37.88 & 2.85 & 0.0043 & 040 & 1.4 & 3 & 5.6 & 91.91 & 4.46 & 0.0095   \\ 
1 & 2.8 & 1 & 5.6 & 11.81 & 1.23 & 0.0031 & 10 & 2.0 & 2 & 5.6 & 44.80 & 3.03 & 0.0029 & 040 & 2.0 & 1 & 4.0 & 53.63 & 1.24 & 0.0043   \\ 
1 & 2.8 & 2 & 4.0 & 14.10 & 2.50 & 0.0078 & 10 & 2.0 & 3 & 4.0 & 46.37 & 4.33 & 0.0091 & 040 & 2.0 & 1 & 5.6 & 63.45 & 1.34 & 0.0022   \\ 
1 & 2.8 & 2 & 5.6 & 16.66 & 2.84 & 0.0031 & 10 & 2.0 & 3 & 5.6 & 54.84 & 4.40 & 0.0084 & 040 & 2.0 & 2 & 4.0 & 75.82 & 2.85 & 0.0042   \\ 
1 & 2.8 & 3 & 4.0 & 17.24 & 4.22 & 0.0041 & 10 & 2.8 & 1 & 4.0 & 31.70 & 1.12 & 0.0052 & 040 & 2.0 & 2 & 5.6 & 89.70 & 3.03 & 0.0029   \\ 
1 & 2.8 & 3 & 5.6 & 20.37 & 4.33 & 0.0064 & 10 & 2.8 & 1 & 5.6 & 37.50 & 1.24 & 0.0031 & 040 & 2.0 & 3 & 4.0 & 92.84 & 4.34 & 0.0091   \\ 
1 & 4.0 & 1 & 5.6 & 14.10 & 1.11 & 0.0038 & 10 & 2.8 & 2 & 4.0 & 44.80 & 2.51 & 0.0077 & 040 & 2.8 & 1 & 4.0 & 63.45 & 1.12 & 0.0052   \\ 
1 & 4.0 & 2 & 5.6 & 19.88 & 2.47 & 0.0057 & 10 & 2.8 & 2 & 5.6 & 52.99 & 2.85 & 0.0031 & 040 & 2.8 & 1 & 5.6 & 75.06 & 1.24 & 0.0031   \\ 
1 & 4.0 & 3 & 5.6 & 24.29 & 4.20 & 0.0026 & 10 & 2.8 & 3 & 4.0 & 54.84 & 4.22 & 0.0043 & 040 & 2.8 & 2 & 4.0 & 89.70 & 2.51 & 0.0077   \\ 
5 & 1.0 & 1 & 2.0 & 09.48 & 1.24 & 0.0086 & 10 & 2.8 & 3 & 5.6 & 64.85 & 4.33 & 0.0065 & 040 & 4.0 & 1 & 5.6 & 89.70 & 1.11 & 0.0038   \\ 
5 & 1.0 & 1 & 4.0 & 13.40 & 1.43 & 0.0017 & 10 & 4.0 & 1 & 5.6 & 44.80 & 1.11 & 0.0038 & 070 & 1.0 & 1 & 2.0 & 35.49 & 1.24 & 0.0086   \\ 
5 & 1.0 & 1 & 5.6 & 15.85 & 1.49 & 0.0006 & 10 & 4.0 & 2 & 5.6 & 63.30 & 2.49 & 0.0056 & 070 & 1.0 & 1 & 4.0 & 50.18 & 1.43 & 0.0017   \\ 
5 & 1.0 & 2 & 2.0 & 13.40 & 2.85 & 0.0085 & 10 & 4.0 & 3 & 5.6 & 77.47 & 4.21 & 0.0028 & 070 & 1.0 & 1 & 5.6 & 59.37 & 1.49 & 0.0006   \\ 
5 & 1.0 & 2 & 4.0 & 18.94 & 3.10 & 0.0073 & 20 & 1.0 & 1 & 2.0 & 18.96 & 1.24 & 0.0086 & 070 & 1.0 & 2 & 2.0 & 50.18 & 2.85 & 0.0084   \\ 
5 & 1.0 & 2 & 5.6 & 22.40 & 3.13 & 0.0069 & 20 & 1.0 & 1 & 4.0 & 26.81 & 1.43 & 0.0017 & 070 & 1.0 & 2 & 4.0 & 70.96 & 3.10 & 0.0073   \\ 
5 & 1.0 & 3 & 5.6 & 27.42 & 4.49 & 0.0097 & 20 & 1.0 & 1 & 5.6 & 31.72 & 1.49 & 0.0006 & 070 & 1.0 & 2 & 5.6 & 83.95 & 3.13 & 0.0069   \\ 
5 & 1.4 & 1 & 4.0 & 15.85 & 1.34 & 0.0029 & 20 & 1.0 & 2 & 2.0 & 26.81 & 2.85 & 0.0085 & 070 & 1.4 & 1 & 4.0 & 59.37 & 1.34 & 0.0029   \\ 
5 & 1.4 & 1 & 5.6 & 18.75 & 1.43 & 0.0012 & 20 & 1.0 & 2 & 4.0 & 37.91 & 3.10 & 0.0073 & 070 & 1.4 & 1 & 5.6 & 70.24 & 1.43 & 0.0012   \\ 
5 & 1.4 & 2 & 4.0 & 22.40 & 3.04 & 0.0043 & 20 & 1.0 & 2 & 5.6 & 44.85 & 3.13 & 0.0069 & 070 & 1.4 & 2 & 4.0 & 83.95 & 3.04 & 0.0043   \\ 
5 & 1.4 & 2 & 5.6 & 26.49 & 3.10 & 0.0052 & 20 & 1.0 & 3 & 5.6 & 54.92 & 4.49 & 0.0097 & 070 & 1.4 & 2 & 5.6 & 99.32 & 3.10 & 0.0052   \\ 
5 & 1.4 & 3 & 2.0 & 19.40 & 4.22 & 0.0086 & 20 & 1.4 & 1 & 4.0 & 31.72 & 1.34 & 0.0029 & 070 & 1.4 & 3 & 2.0 & 72.71 & 4.23 & 0.0087   \\ 
5 & 1.4 & 3 & 5.6 & 32.43 & 4.46 & 0.0094 & 20 & 1.4 & 1 & 5.6 & 37.53 & 1.43 & 0.0012 & 070 & 2.0 & 1 & 4.0 & 70.96 & 1.24 & 0.0043   \\ 
5 & 2.0 & 1 & 4.0 & 18.94 & 1.24 & 0.0043 & 20 & 1.4 & 2 & 4.0 & 44.85 & 3.04 & 0.0043 & 070 & 2.0 & 1 & 5.6 & 83.95 & 1.34 & 0.0022   \\ 
5 & 2.0 & 1 & 5.6 & 22.40 & 1.34 & 0.0022 & 20 & 1.4 & 2 & 5.6 & 53.06 & 3.10 & 0.0052 & 070 & 2.8 & 1 & 4.0 & 83.95 & 1.12 & 0.0052   \\ 
5 & 2.0 & 2 & 4.0 & 26.76 & 2.85 & 0.0043 & 20 & 1.4 & 3 & 2.0 & 38.85 & 4.23 & 0.0087 & 070 & 2.8 & 1 & 5.6 & 99.32 & 1.24 & 0.0031   \\ 
5 & 2.0 & 2 & 5.6 & 31.65 & 3.03 & 0.0029 & 20 & 1.4 & 3 & 5.6 & 64.96 & 4.46 & 0.0095 & 100 & 1.0 & 1 & 2.0 & 42.42 & 1.24 & 0.0086   \\ 
5 & 2.0 & 3 & 4.0 & 32.76 & 4.33 & 0.0091 & 20 & 2.0 & 1 & 4.0 & 37.91 & 1.24 & 0.0043 & 100 & 1.0 & 1 & 4.0 & 59.98 & 1.43 & 0.0017   \\ 
5 & 2.0 & 3 & 5.6 & 38.73 & 4.40 & 0.0084 & 20 & 2.0 & 1 & 5.6 & 44.85 & 1.34 & 0.0022 & 100 & 1.0 & 1 & 5.6 & 70.97 & 1.49 & 0.0006   \\ 
5 & 2.8 & 1 & 4.0 & 22.40 & 1.12 & 0.0052 & 20 & 2.0 & 2 & 4.0 & 53.59 & 2.85 & 0.0042 & 100 & 1.0 & 2 & 2.0 & 59.98 & 2.85 & 0.0084   \\ 
5 & 2.8 & 1 & 5.6 & 26.49 & 1.24 & 0.0031 & 20 & 2.0 & 2 & 5.6 & 63.40 & 3.03 & 0.0029 & 100 & 1.0 & 2 & 4.0 & 84.82 & 3.10 & 0.0073   \\ 
5 & 2.8 & 2 & 4.0 & 31.65 & 2.51 & 0.0077 & 20 & 2.0 & 3 & 4.0 & 65.62 & 4.34 & 0.0091 & 100 & 1.4 & 1 & 4.0 & 70.97 & 1.34 & 0.0029   \\ 
5 & 2.8 & 2 & 5.6 & 37.43 & 2.84 & 0.0031 & 20 & 2.0 & 3 & 5.6 & 77.62 & 4.40 & 0.0084 & 100 & 1.4 & 1 & 5.6 & 83.96 & 1.43 & 0.0012   \\ 
5 & 2.8 & 3 & 4.0 & 38.73 & 4.22 & 0.0042 & 20 & 2.8 & 1 & 4.0 & 44.85 & 1.12 & 0.0052 & 100 & 1.4 & 3 & 2.0 & 86.91 & 4.23 & 0.0087   \\ 
5 & 2.8 & 3 & 5.6 & 45.80 & 4.33 & 0.0065 & 20 & 2.8 & 1 & 5.6 & 53.06 & 1.24 & 0.0031 & 100 & 2.0 & 1 & 4.0 & 84.82 & 1.24 & 0.0043   \\   \hline
\end{tabular}
\end{table*}

\end{document}